\documentclass{article} % For LaTeX2e
\usepackage{iclr2026_conference,times}

\usepackage{amsmath,amsfonts,bm}
\usepackage{amssymb}
\usepackage{mathtools}
\usepackage{mathrsfs}
\usepackage{amsthm}
\usepackage{multirow}
\usepackage{booktabs}
\usepackage{pifont}

\theoremstyle{plain}
\newtheorem{theorem}{Theorem}[section]
\newtheorem{proposition}{Proposition}[section]
\newtheorem{lemma}{Lemma}[section]

\theoremstyle{definition}
\newtheorem{definition}{Definition}[section]

\theoremstyle{remark}
\newtheorem{remark}{Remark}[section]

\def\eqref#1{equation~\ref{#1}}
\def\1{\bm{1}}

\DeclareMathAlphabet{\mathsfit}{\encodingdefault}{\sfdefault}{m}{sl}
\SetMathAlphabet{\mathsfit}{bold}{\encodingdefault}{\sfdefault}{bx}{n}

\def\sR{{\mathbb{R}}}

\newcommand{\R}{\mathbb{R}}

\usepackage{hyperref}
\usepackage{url}
\usepackage{graphicx}
\usepackage{cleveref}

\usepackage[table,xcdraw]{xcolor}
\usepackage{booktabs}

\usepackage{url}

\usepackage[most]{tcolorbox} 

\definecolor{lightgreen}{RGB}{220, 255, 220}
\definecolor{lightpink}{RGB}{255, 220, 230}

\usepackage{tikz}
\usetikzlibrary{arrows.meta,positioning,calc}
\usepackage{subcaption} % for subfigure
\usepackage[framemethod=TikZ]{mdframed}
\usepackage[framemethod=TikZ]{mdframed}
\definecolor{lightGray}{gray}{0.97}
\definecolor{midGray}{gray}{0.40}
\definecolor{lightYellow}{RGB}{254,254,239}
\definecolor{darkGray}{rgb}{0.45,0.45,0.45}
\definecolor{darkerGray}{rgb}{0.3,0.3,0.3}
\definecolor{purpleblue}{RGB}{90,96,136}
\mdfdefinestyle{myFrameStyle}{%
  linecolor=purpleblue,
  linewidth=0.7pt,
  roundcorner=3pt,
  skipabove=5pt,
  skipbelow=0pt,
  innertopmargin=5pt,
  innerbottommargin=5pt,
  innerrightmargin=5pt,
  innerleftmargin=5pt,
  leftmargin=0pt,
  rightmargin=0pt,
}

\newcommand{\myFrameblue}[1]{
    \vspace{4pt}
    \begin{mdframed}[style=myFrameStyle,userdefinedwidth=
    \linewidth,align=center,skipabove=6pt,skipbelow=0pt, 
    backgroundcolor=blue!10]
    {#1}
    \end{mdframed}\vspace{-8pt}
}

\newcommand{\myFramegreen}[1]{
    \vspace{4pt}
    \begin{mdframed}[style=myFrameStyle,userdefinedwidth=
    \linewidth,align=center,skipabove=6pt,skipbelow=0pt, 
    backgroundcolor=green!10]
    {#1}
    \end{mdframed}\vspace{-8pt}
}

\newcommand{\myFrameorange}[1]{
    \vspace{4pt}
    \begin{mdframed}[style=myFrameStyle,userdefinedwidth=
    \linewidth,align=center,skipabove=6pt,skipbelow=0pt, 
    backgroundcolor=orange!10]
    {#1}
    \end{mdframed}\vspace{-8pt}
}

\newcommand{\myFramemagenta}[1]{
    \vspace{4pt}
    \begin{mdframed}[style=myFrameStyle,userdefinedwidth=
    \linewidth,align=center,skipabove=6pt,skipbelow=0pt, 
    backgroundcolor=magenta!20]
    {#1}
    \end{mdframed}\vspace{-8pt}
}

\title{Conditioned Initialization for Attention}
\author{Hemanth Saratchandran \\
Australian Institute for Machine Learning \\ Adelaide University \\
\And
Simon Lucey \\
Australian Institute for Machine Learning \\ Adelaide University}

\iclrfinalcopy % Uncomment for camera-ready version, but NOT for submission.
\begin{document}

\maketitle

\begin{abstract}
Transformers are a dominant architecture in modern machine learning, powering applications across vision, language, and beyond. At the core of their success lies the attention layer, where the query, key, and value matrices determine how token dependencies are captured. While considerable work has focused on scaling and optimizing Transformers, comparatively little attention has been paid to how the weights of the queries, keys and values are initialized. Common practice relies on random initialization or alternatives such as mimetic initialization, which imitates weight patterns from converged models, and weight selection, which transfers weights from a teacher model. In this paper, we argue that initialization can introduce an optimization bias that fundamentally shapes training dynamics. We propose \textbf{conditioned initialization}, a principled scheme that initializes attention weights to improve the spectral properties of the attention layer. Theoretically, we show that conditioned initialization can potentially reduce the condition number of the attention Jacobian, leading to more stable optimization. Empirically, it accelerates convergence and improves generalization across diverse applications, highlighting conditioning as a critical yet underexplored area for advancing Transformer performance. Importantly, conditioned initialization is simple to apply and integrates seamlessly into a wide range of Transformer architectures.

\end{abstract}

\section{Introduction}\label{sec:intro}

Transformers~\citep{vaswani2017attention} have rapidly become a cornerstone of modern machine learning, driving progress in fields as diverse as natural language processing~\citep{vaswani2017attention, zhuang2021robustly, zhen2022cosformer}, computer vision~\citep{dosovitskiy2020image, liu2021swin, touvron2021training, carion2020end}, and robotics~\citep{salzmann2020trajectron++, maiti2023transfusion}. 
A key factor behind this versatility is the self-attention mechanism, which models interactions between tokens by comparing them pairwise and dynamically weighting their contributions. 
This ability to capture both long-range and global dependencies has positioned Transformers as a foundational architecture across a wide spectrum of learning tasks.  

While substantial effort has been devoted to improving the efficiency, scalability, and expressiveness of attention mechanisms~\citep{ali2021xcit, xiong2021nystromformer, ding2022davit}, relatively little focus has been placed on a more basic but equally important aspect: \emph{how attention weights are initialized}. 
Initialization plays a critical role in shaping optimization landscapes of deep neural networks. Classical schemes such as Xavier \citep{glorot2010understanding} and Kaiming \citep{he2015delving} initialization showed that carefully chosen scaling of weights at the start of training can dramatically improve gradient optimization and stability in deep networks.
These insights were crucial for enabling the training of very deep architectures such as ResNets \citep{he2016deep}, where poor initialization could otherwise cause vanishing or exploding gradients. Yet, despite their depth and complexity, Transformers have received far less theoretical scrutiny on this front. Given that self-attention relies on query, key, and value projections whose interplay directly governs the stability of token interactions, it is natural to ask whether initialization schemes tailored specifically to attention could offer similar benefits.

Currently, Transformers typically adopt simple random initializations, without consideration of the unique structure of attention layers. 
Recent alternatives such as \emph{mimetic initialization}~\citep{trockman2023mimetic}, which transfers statistical patterns from converged models, and \emph{weight selection}~\citep{xu2023initializing}, which reuses pretrained weights from larger teacher models, highlight a growing recognition that initialization matters. 
However, these methods remain heuristic and lack a principled connection to the conditioning of the attention mechanism itself.  

In this work, we revisit Transformer initialization from a theoretical perspective. 
We show that the stability of optimization in self-attention layers is closely tied to the conditioning of their Jacobians, which in turn depends on the spectral properties of the query, key, and value projections. 
Building on this insight, we introduce \textbf{conditioned initialization}, a principled scheme designed to improve the spectral conditioning of attention blocks at the start of training. 
Rather than directly modifying training objectives, our method intervenes at initialization, providing an inductive bias that promotes more stable optimization dynamics.  

Our contributions are threefold:  
\begin{enumerate}
    \item \textbf{Theoretical framework:} We establish a connection between the conditioning of self-attention Jacobians and the spectral structure of the query, key, and value matrices, motivating initialization schemes that explicitly target this property.
    \item \textbf{Conditioned initialization:} We propose a simple initialization method that reduces the upper bound on the attention Jacobian’s condition number, thereby biasing training toward stable optimization.  
    \item \textbf{Empirical validation:} Through experiments on diverse benchmarks, spanning image classification, object detection, instance segmentation, language modeling, and long-range sequence learning, we show that conditioned initialization consistently accelerates convergence and improves generalization, and can be easily integrated into a variety of different Transformer architectures. 
\end{enumerate}

By highlighting initialization as a critical yet underexplored component of Transformer design, our work opens up a new perspective on how principled conditioning can be leveraged to improve optimization stability and downstream performance.

\section{Related Work}\label{sec:rel_work}

\paragraph{Initialization.} Initialization strategies play a pivotal role in determining how efficiently deep networks can be trained. Early advances such as Xavier~\citep{glorot2010understanding} and Kaiming initialization~\citep{he2015delving} demonstrated that properly scaling weights at the outset helps preserve variance across layers, preventing issues like vanishing or exploding gradients. These ideas were foundational in enabling the successful training of very deep architectures, most notably ResNets~\citep{he2016deep}. In the context of Transformers, however, initialization has typically been treated in a more ad hoc manner, with standard practice relying on simple schemes from normal or truncated normal distributions. More recent efforts have begun to acknowledge the unique structure of attention layers. For example, mimetic initialization~\citep{trockman2023mimetic} introduces inductive bias by imitating the statistical patterns of trained networks, while weight selection~\citep{xu2023initializing} transfers weights from larger pretrained teacher models to provide a stronger starting point for smaller architectures. In our experiments we compare to these two methods.

%These approaches highlight a growing awareness that initialization can substantially influence optimization and generalization in Transformers, but they remain heuristic and lack a principled connection to the conditioning of attention mechanisms.

\paragraph{Conditioning.}
A growing body of work has underscored the importance of conditioning for both the trainability and generalization of neural networks. In particular, ~\cite{saratchandran2025weight} showed that networks with better-conditioned weights tend to achieve superior performance, and proposed a matrix preconditioning method to explicitly control condition numbers during training. From a different angle, \cite{liu2022loss} analyzed optimization through the lens of the neural tangent kernel (NTK), demonstrating that well-conditioned NTKs lead to faster and more reliable convergence, particularly in the infinite-width regime where the NTK dominates learning dynamics~\citep{jacot2018neural}. This has been extended to the transformer setting in \cite{yang2020tensor}. 
Other studies have pointed out structural factors that affect conditioning. For example, \cite{agarwal2021deep} observed that increasing network depth can itself improve conditioning, thereby aiding gradient-based methods and \cite{macdonald2023skip} showed how skip connections yield a conditioned loss landscape. In the case of Transformers, \cite{ji2025always} argued that skip connections act as an implicit conditioning mechanism, stabilizing the optimization of deep Transformers. Furtherworks have shown in a variety of ways how conditioning transformers through architectural modifications can lead to better optimization and performance \cite{saratchandran2026spectral, albert2025towards, saratchandran2025leaner, saratchandran2025enhancing}. Our work departs from these approaches by asking whether initialization itself can be designed to yield better-conditioned attention layers from the outset.

\section{Theoretical Framework}\label{sec:theory}

\subsection{Preliminaries}\label{sec:prelims}

For the theoretical framework we will primarily focus on self-attention, which is one of the most common forms of attention in a Transformer. Self-attention is composed of three learnable matrices, query $W_Q \in \R^{D\times d}$, key $W_K \in \R^{D \times d}$, and value 
$W_V \in \R^{D\times d}$ defined for an input sequence $X \in \mathbb{R}^{N \times D}$.
\begin{equation}\label{eqn:attn_eqn_general}
    \mathrm{A}(X) = \mathrm{softmax}(XW_QW_K^TX^T)XW_V 
\end{equation}
where $\mathrm{softmax}$ is the softmax activation that acts row-wise on a matrix \citep{prince2023understanding}. Note that then $\mathrm{A}(X) \in \R^{N\times d}$. In general, Transformers employ multiple heads i.e. attention matrices 
$\mathrm{A}_i$ for $1 \leq i \leq h$ where $h$ is the number of heads. These are then concatenated together to form a multi-head attention layer $[A_1,\ldots ,A_h]$. We point out that in some references the notation $\mathrm{A}(X)$ is reserved for only the term $\mathrm{softmax}(XW_QW_K^TX^T)$. However, in this paper as we will be concerned with the whole attention layer we use $\mathrm{A}(X)$ to denote the whole layer output as defined in \cref{eqn:attn_eqn_general}.
For further details on Transformers readers may consult \cite{prince2023understanding}.

%When training a Transformer the self-attention parameters are given by the weight matrices $W_Q$, $W_K$, $W_V$ and $W_P$. 

The self-attention map of a layer in a Transformer $\mathrm{A}(X)$ has parameters given by those parameters in $X$ from the previous layer and those given by $W_Q$, $W_K$, $W_V$ that define $\mathrm{A}(X)$. Our work will consider the Jacobian of $\mathrm{A}(X)$ with respect to the parameters within the layer of $\mathrm{A}(X)$, namely $W_Q$, $W_K$, $W_V$.
Therefore, when we speak of the Jacobian of $\mathrm{A}(X)$ it will be with respect to $W_Q$, $W_K$, $W_V$. We will denote this Jacobian by 
$\mathrm{J}(\mathrm{A}(X))$ and note that it is defined by
\begin{equation}\label{eqn:jac_defn}
 \mathrm{J}(\mathrm{A}(X)) =  
 \bigg{[}\frac{\partial{\mathrm{A}(X)}}{\partial{W_Q}}, 
 \frac{\partial{\mathrm{A}(X)}}{\partial{W_K}}, 
 \frac{\partial{\mathrm{A}(X)}}{\partial{W_V}}\bigg{]}^T  
\end{equation}

Given a matrix $Z \in \R^{m\times n}$ we denote the vectorization of $Z$ by $\mathrm{vec}(Z) \in \R^{mn\times 1}$ \citep{magnus2019matrix}. Note that for such a matrix
there is a transformation $T_{mn} \in \R^{mn \times mn}$ such that
$T_{mn}\mathrm{vec}(Z) = \mathrm{vec}(Z^T)$ where $Z^T$ denotes the transpose of $Z$. The matrix $T_{mn}$ is known as a commutation matrix and is a permutation matrix \citep{magnus2019matrix}. The maximum singular value of a matrix $Z$ will be denoted by $\sigma_{\max}(Z)$ and the minimum singular value by $\sigma_{\min}(Z)$. We will use the standard terminology SVD to denote the singular value decomposition of a matrix. Given a vector $v \in \R^n$ the notation $||v||_2$ will denote the vector 2-norm of $v$. Finally, we will let $I_{m \times n}$ denote the rectangular identity matrix that has all $1$'s on its main diagonal and 
$\mathcal{O}_{m \times n}$ as the real $m \times n$ semi-orthogonal matrices.

\subsection{Main Theorems}\label{sec:main_theory}

In this section, we present the main theorem of the paper, which motivates the development of a simple yet effective strategy, conditioned initialization, designed to reduce the condition number of the Jacobian of the attention layer at initialization. As shown in \cref{sec:exps}, this scheme is straightforward to implement while providing significant optimization benefits.

\begin{definition}
Let $Z$ be an $N \times d$ matrix of full rank. The condition number of $Z$, denoted by $\kappa$, is defined as
\begin{equation}
    \kappa(Z) = \frac{\sigma_{\max}(Z)}{\sigma_{\min}(Z)}
\end{equation}
where $\sigma_{\max}(Z)$ denotes the maximum singular value of $Z$ and $\sigma_{\min}(Z)$ the minimum singular value of $Z$, which we know is non-zero as $Z$ is full rank.
\end{definition}

Our objective is to analyze the condition number of the self-attention layer in a Transformer. 
We show that the condition number of its Jacobian depends on the condition numbers of the 
query, key, and value weight matrices. Furthermore, we demonstrate that initializing 
$W_Q$, $W_K$, and $W_V$ with low condition numbers imparts an inductive bias into the 
Transformer architecture that leads to more effective optimization.

We begin by examining the derivatives of the self-attention layer with respect to the parameters $W_Q$, $W_K$, and $W_V$. We will need the following lemma.

\myFramemagenta{
\begin{lemma}\label{lem:softmax_deriv}
Let $\Lambda : \R^n \rightarrow \R^{n\times n}$ denote the function 
$\Lambda(z) = Diag(z) - z\cdot z^T$.
We then have that 
\begin{equation}
\frac{\partial \mathrm{softmax}}{\partial{x}}(z) = 
\Lambda(\mathrm{softmax}(z)).
\end{equation}
\end{lemma}}

\myFrameblue{
\begin{proposition}\label{thm:attn_derivs}
Let $\mathrm{A}(X)$ denote a self-attention matrix with input $X$ as defined by \cref{eqn:attn_eqn_general}. Then
\begin{align}
 \frac{\partial{\mathrm{A}(X)}}{\partial{W_Q}} &= 
 (W_V^TX^T\otimes I_{N})\bigg{(} 
 \Lambda(\mathrm{softmax}(XW_QW_K^TX^T))
 \bigg{)}(XW_K\otimes X) \\
\frac{\partial{\mathrm{A}(X)}}{\partial{W_K}} &= 
(W_V^TX^T\otimes I_N)\bigg{(} 
 \Lambda(\mathrm{softmax}(XW_QW_K^TX^T))
 \bigg{)}(X\otimes XW_Q)\cdot T_{Dd} \\
\frac{\partial{\mathrm{A}(X)}}{\partial{W_V}} &= 
I_d\otimes \mathrm{softmax}(XW_QW_K^TX^T)X
\end{align}
where $\Lambda$ is defined in \cref{lem:softmax_deriv} and $T_{Dd}$ is the commutation matrix satisfying 
$T_{Dd}\mathrm{vec}(W_K) = \mathrm{vec}(W_K^T)$ (see \cref{sec:prelims}).
\end{proposition}}

From \cref{thm:attn_derivs} we obtain a bound on the condition number of the Jacobian of the attention matrix.

\myFrameblue{
\begin{theorem}\label{thm:cond_jac}
Let $\mathrm{A}(X)$ denote a self-attention matrix, as defined in 
\cref{eqn:attn_eqn_general}, with input $X$ and let 
$J(\mathrm{A}(X))$ denote its Jacobian with respect to the parameter matrices $W_Q$, $W_K$ and $W_V$ as defined in 
\cref{eqn:jac_defn}. Assume that $J(\mathrm{A}(X))$ has full rank so that $\kappa(J(\mathrm{A}(X)))$ is finite.
Then 
\begin{align}
  \kappa(J(\mathrm{A}(X))) \leq 
\kappa(X)^3
&\kappa\big{(}\Lambda(\mathrm{softmax}(XW_QW_K^TX^T))\big{)}\kappa(W_V)\big{(} 
\kappa(W_Q) + \kappa(W_K)
\big{)} \label{eqn:cond_jac}\\
&+ \kappa(X)\kappa(\mathrm{softmax}(XW_QW_K^TX^T)) \nonumber
\end{align}
where $\Lambda$ is defined in \cref{lem:softmax_deriv}.
\end{theorem}}

\Cref{thm:cond_jac} shows that
$\kappa(J(\mathrm{A}(X)))$ is bounded above by a sum of two terms:
\begin{align}
&\kappa(X)^3
\cdot\kappa\big{(}\Lambda(\mathrm{softmax}(XW_QW_K^TX^T))\big{)}\kappa(W_V)\big{(} 
\kappa(W_Q) + \kappa(W_K)
\big{)} \label{first_term}\\
&\kappa\big(\mathrm{softmax}(XW_QW_K^T X^T)\big) \label{second_term}
\end{align}

\paragraph{Observation.} \Cref{thm:cond_jac} provides a strategy for reducing the condition number of the Jacobian of $\mathrm{A}$ through the upper bound in \cref{eqn:cond_jac}. Since we directly control $W_Q$, $W_K$, and $W_V$, reducing their condition numbers decreases the term in \cref{first_term}, thereby tightening the bound. Our next step is to show that this can be achieved at initialization and to confirm empirically in \cref{sec:exps} that it introduces a bias which improves both optimization and performance.

\subsection{Conditioned Initialization}\label{sec:spec_cond_init}

Our goal in this section is to design a simple and effective initialization for the query, key, and value matrices $W_Q$, $W_K$, and $W_V$ that lowers the condition number of their singular value spectra, thereby improving the conditioning of the Jacobian $\mathrm{J}(\mathrm{A})$.

We begin with two key observations. There are two families of $m \times n$ matrices with condition number~$1$:
\begin{itemize}
    \item[1.] scalar multiples of the identity $\lambda I_{m \times n}$ with $\lambda \neq 0$, and
    \item[2.] semi-orthogonal matrices $\mathcal{O}_{m \times n}$ (matrices with orthonormal rows or columns).
\end{itemize}

From \cref{thm:cond_jac}, the condition number of $\mathrm{J}(\mathrm{A})$ admits the surrogate upper bound
\begin{align}
    \mathcal{B}(\mathrm{J}(\mathrm{A})) 
    &:= \kappa(X)^3 \,
       \kappa\!\left(\Lambda(\mathrm{softmax}(XW_QW_K^T X^T))\right)
       \kappa(W_V)\,\big(\kappa(W_Q)+\kappa(W_K)\big) \\
    &\quad+ \kappa(X)\,\kappa(\mathrm{softmax}(XW_QW_K^T X^T)).
\end{align}
Unlike the true condition number $\kappa(\mathrm{J}(\mathrm{A}))$, the bound $\mathcal{B}(\mathrm{J}(\mathrm{A}))$ can be directly influenced at initialization by controlling the conditioning of $W_Q$, $W_K$, and $W_V$. Standard practice initializes these matrices from Gaussian or uniform distributions, which do not enforce good conditioning.

The following proposition shows that initializing $W_Q$, $W_K$, and $W_V$ from either of the families above yields a strictly better surrogate bound at initialization.

\myFramegreen{
\begin{proposition}\label{prop:u_bound_init}
Let $\mathrm{A}(X)$ denote an attention matrix with $W_Q$, $W_K$, and $W_V$ initialized from a Gaussian or uniform distribution, and let $\overline{\mathrm{A}}(X)$ denote one where $\overline{W}_Q$, $\overline{W}_K$, and $\overline{W}_V$ are initialized from either $\{\lambda I_{D \times d}\}$ or $\mathcal{O}_{D \times d}$. Then
\[
    \mathcal{B}(\mathrm{J}(\overline{\mathrm{A}})) \;\leq\; \mathcal{B}(\mathrm{J}(\mathrm{A})).
\]
\end{proposition}}

\paragraph{Remark.} The quantity $\mathcal{B}(\mathrm{J}(\mathrm{A}))$ is only an upper bound on $\kappa(\mathrm{J}(\mathrm{A}))$. Thus, \cref{prop:u_bound_init} guarantees a tighter bound but does not by itself imply improved conditioning of the Jacobian. Nonetheless, as we demonstrate in \cref{sec:exps}, the initialization in \cref{prop:u_bound_init} consistently lowers the condition number during training and leads to more stable optimization and improved performance.

\paragraph{Initialization Strategy.}  
Although \cref{prop:u_bound_init} shows that initializing the matrices $W_Q$, $W_K$, and $W_V$ using either $\{\lambda I_{D \times d}\}$ or $\mathcal{O}_{D \times d}$ can potentially lower the condition number of $\mathrm{J}(\mathrm{A})$, it does not specify which choice is most suitable for each matrix. We note that $W_Q$, $W_K$, and $W_V$ play distinct algebraic roles in attention, and this motivates different treatments. For the value map $W_V$, which enters linearly into the output
\[
(\mathrm{softmax}(XW_QW_K^T X^T))(XW_V),
\]
initializing with the rectangular identity preserves the scale of the input representations ($XW_V = X$), keeps $\kappa(W_V)=1$, and avoids unnecessary distortion of the Jacobian. In contrast, the query and key maps interact bilinearly through
\[
S = XW_QW_K^T X^T,
\]
and initializing them as rectangular identities can bias projections toward coordinate subspaces, yielding anisotropic logits and unstable softmax dynamics. A semi-orthogonal initialization for $W_Q$ and $W_K$ instead provides near-isometric embeddings, giving each head balanced representations of $X$ and supporting more diverse and stable attention patterns.

\myFrameorange{
\paragraph{Implementation.} 
Following the above design principle, in the experiments of \cref{sec:exps} we initialized the value matrices $W_V$ for each head as rectangular identities. For the queries and keys, we initialized each $W_Q^{(i)}$ and $W_K^{(i)}$ in the $i$-th head with independent semi-orthogonal projections,
\[
(W_Q^{(i)})^T W_Q^{(i)} = I_d, 
\qquad 
(W_K^{(i)})^T W_K^{(i)} = I_d,
\]
for $i=1,\dots,h$, where $h$ is the number of heads. This produces near-isometric embeddings into distinct subspaces, thereby diversifying the logits $S^{(i)} = (XW_Q^{(i)})(XW_K^{(i)})^T$ and the resulting attention patterns. See \cref{sec:imp} for a concrete way to carry out this procedure.
We refer to this initialization as \textbf{conditioned initialization}.
}
%In PyTorch, this was implemented using independent \texttt{nn.init.orthogonal} calls per head with different random seeds. 

\paragraph{Different forms of attention.} 
The formulation in \cref{sec:prelims} describes the classical self-attention layer used in Transformers. 
In practice, many recent architectures have proposed variations of attention to improve efficiency and effectiveness \citep{touvron2021training, ali2021xcit, liu2021swin, ding2022davit, xiong2021nystromformer}. In  particular, many of these newer forms of attention apply normalization to the query, key and values.
Our conditioned initialization is readily applicable to these generalized forms of attention, including those with normalization (see \cite{henry2020query, dehghani2023scaling, zhang2019root}), and we empirically demonstrate in \cref{sec:exps} that it consistently yields strong performance.

\begin{figure}[ht!]
    \centering
    \includegraphics[width=0.48\linewidth]{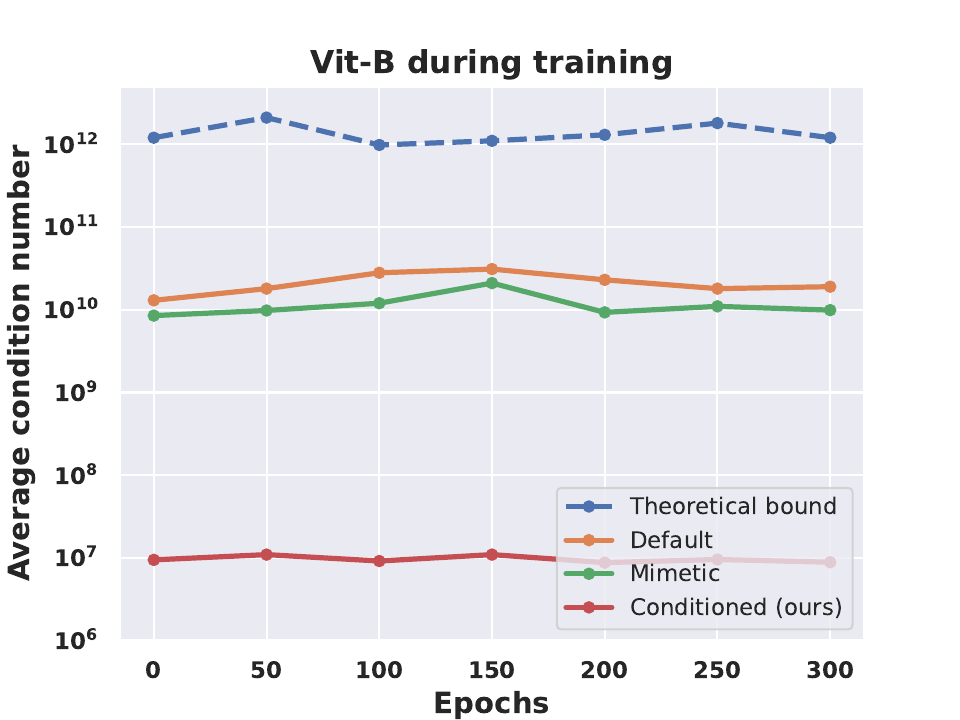}
    \hfill
    \includegraphics[width=0.48\linewidth]{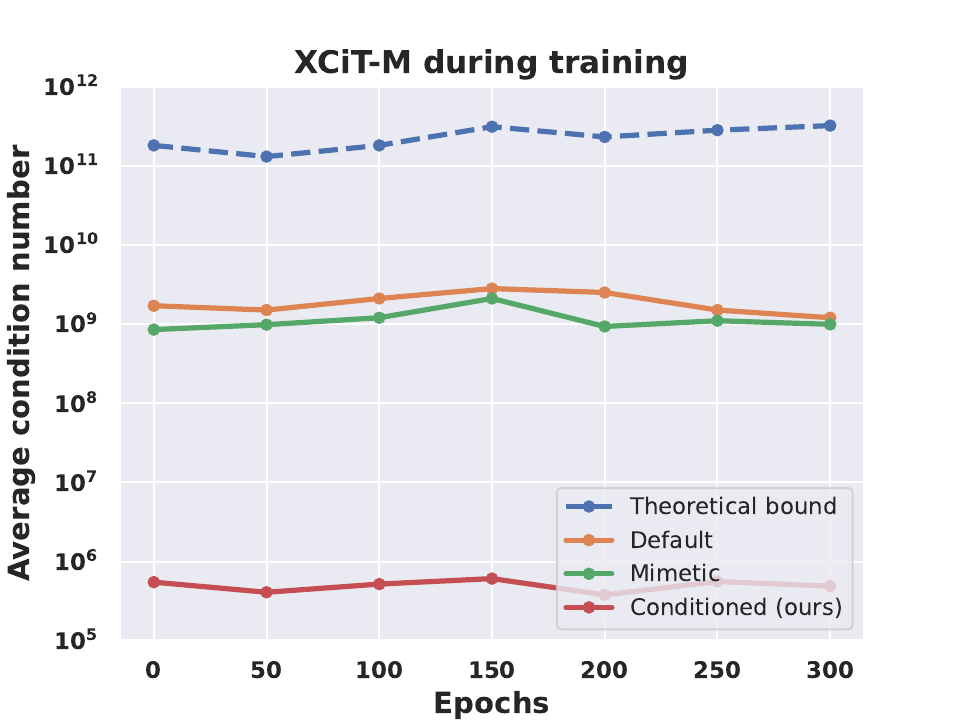}
    \caption{Average condition number of the attention Jacobian during training under three common initialization schemes, shown alongside the theoretical bound from \cref{eqn:cond_jac}.}
    \label{fig:vit_train_analysis}
\end{figure}

\section{Experiments}\label{sec:exps}

In this section, we evaluate the theoretical insights from \cref{sec:theory} across a range of Transformer applications. 
For each setting, we compare our conditioned initialization against the standard default schemes commonly used in the literature, as well as more recent alternatives such as mimetic initialization~\citep{trockman2023mimetic} and weight selection~\citep{xu2023initializing}. 
In all experiments, conditioned initialization is implemented as described in \cref{sec:spec_cond_init}. The goal is to test our initialization on a wide range of architectures at different parameter levels. For larger scale experiments we refer the reader to \cref{app:large_exps}.

\subsection{ViTs for Image Classification}\label{subsec:IC}

\paragraph{Vision Transformers.} We applied conditioned initialization to the attention layer of a variety of modern vision Transformers: a ViT-Base (ViT-B) \citep{dosovitskiy2020image}, a Swin-Base (Swin-B) \citep{liu2021swin}, a XCiT-Medium (XCiT-M) \citep{ali2021xcit}, a DeiT-Base (DeiT-B) \citep{touvron2021training}, and a DaViT-Base (DaViT-B) \citep{ding2022davit} for image classification on ImageNet-1k. 
Each model is initialized in three ways: (i) the default truncated normal initialization~\citep{timm2025}, 
(ii) mimetic initialization~\citep{trockman2023mimetic}, and (iii) our conditioned initialization from \cref{sec:spec_cond_init}. 
We point out that each of these vision Transformers uses a different attention layers to the standard self-attention used in the ViT-B architecture. However, as mentioned in \cref{sec:spec_cond_init} conditioned initialization applies to more general forms of attention.

\paragraph{Results on ImageNet-1k.} We trained three versions of each Transformer: a baseline model with a default initialization \citep{timm2025}, one with mimetic initialization \citep{trockman2023mimetic} and one
incorporating conditioned initialization, see \cref{app:vision_trans} for training details.
The final results, summarized in \cref{tab:vits}, show the test accuracy for each initialization on each model. We observe that in every case, conditioned initialization outperforms the other two.

\paragraph{Validating the theory.} 
We validate the theoretical results of \cref{sec:theory} using ViT-B and XCiT-M models trained on ImageNet-1k. 
\Cref{fig:vit_train_analysis} reports the average condition number of the Jacobian of the attention matrices 
for ViT-B (left) and XCiT-M (right), during training, under the above mentioned three initializations, alongside the theoretical upper bound from \cref{thm:cond_jac}. 
The results demonstrate that conditioned initialization consistently yields a better-conditioned Jacobian, 
providing empirical support for its role in enabling more stable attention mechanisms.

\begin{table*}[!ht]
\caption{Comparison of Vision Transformers with different initializations pretrained on ImageNet-1k. We report Top-1\% classification accuracy. In each case, conditioned initialization improves performance over the default and mimetic initializations.}
\label{tab:vits}
\centering
\begin{tabular}{c|c c c c c}
    \toprule
    \rowcolor{gray!10}
    & ViT-B & DeiT-B & Swin-B & XCiT-M & DaViT-B \\
    \midrule
    Original & 80.3  & 81.6 & 83.4 & 82.6 & 84.3 \\
    \midrule
    \rowcolor{lightpink}
    Mimetic & 80.5  & 81.6 & 83.5 & 82.6 & 84.4 \\
    \midrule
    \rowcolor{lightgreen}
    Conditioned (ours) & 81.5 & 82.7  & 84.6  & 83.5  & 85.3 \\
    \bottomrule
\end{tabular}
\end{table*}

\begin{figure}[ht!]
    \centering
    \includegraphics[width=0.48\linewidth, height=0.22\textheight]{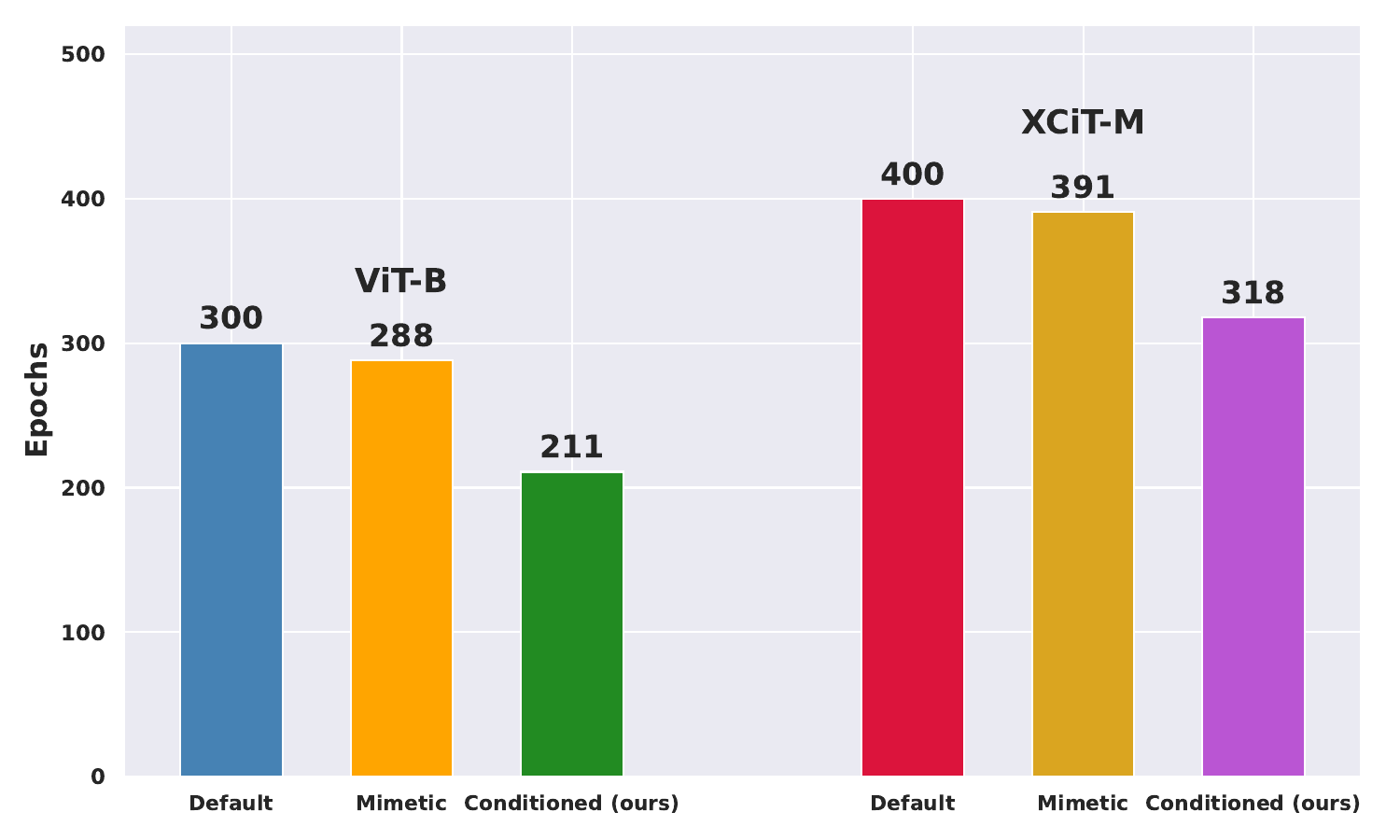}
    \hfill
    \includegraphics[width=0.48\linewidth, height=0.22\textheight]{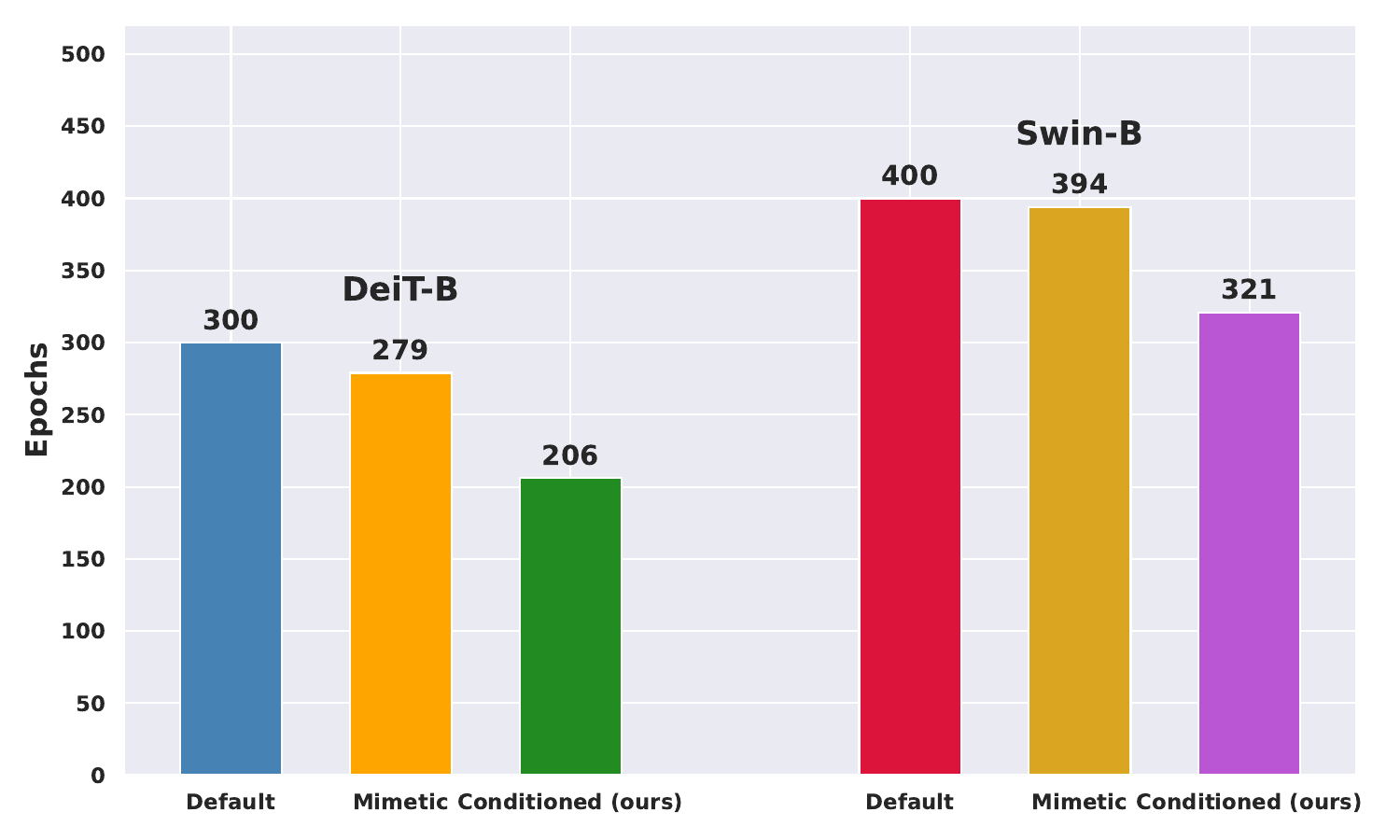}
    \caption{Total number of epochs required for each initialization to reach the final accuracy of the default initialization reported in \cref{tab:vits}, across ViT-B, DeiT-B, Swin-B, and XCiT-M. In all cases, conditioned initialization converges more quickly, requiring fewer epochs than both default and mimetic initialization.}
    \label{fig:vit_converge}
\end{figure}

\paragraph{Optimization efficiency.} 
As shown in \cref{tab:vits}, conditioned initialization achieves higher accuracy under the same training regime compared to both default and mimetic initialization. 
To further assess training efficiency, we measured the number of epochs required by mimetic and conditioned initialization to match the final accuracy attained by the default initialization across ViT-B, DeiT-B, Swin-B, and XCiT-M. 
\Cref{fig:vit_converge} reports these results, demonstrating that conditioned initialization consistently converges 20–30\% faster to the same accuracy level. Similar analysis for DaViT-B is given in \cref{app:vision_trans}.

\subsubsection{Small scale datasets}\label{sec:small_scale}

Following \cref{subsec:IC}, we assessed conditioned initialization on small-scale image classification datasets: Flowers~\citep{nilsback2008automated}, Pets~\citep{vedaldi2012cats}, CIFAR-10, and CIFAR-100~\citep{krizhevsky2009learning}. 
Experiments were conducted with the ViT-Tiny (ViT-T) architecture, a common choice for such benchmarks~\citep{trockman2023mimetic, xu2023initializing}. 
We compared conditioned initialization with the default truncated normal scheme~\citep{timm2025} and with mimetic initialization~\citep{trockman2023mimetic}. 
Because ViTs lack strong inductive bias, they typically perform poorly on small datasets; mimetic initialization has been shown to partially remedy this by providing a more suitable inductive bias. 
As shown in \cref{tab:small_vits}, conditioned initialization reliably improves over the default and performs on par with mimetic initialization.

\paragraph{Weight selection.} We compared our initialization with the weight selection method from \cite{xu2023initializing}. This can be found in \cref{app:vision_trans}.

\begin{table*}[!ht]
\caption{Comparison of ViT-T pretrained on four small scale datasets with three different initializations. We report Top-1\% classification accuracy. In each case, conditioned initialization improves performance over the default and mimetic initializations. }
\label{tab:small_vits}
\centering
\begin{tabular}{c|c c c c}
    \toprule
    \rowcolor{gray!10}
    & Pets & Flowers & CIFAR-10 & CIFAR-100 \\
    \midrule
    Default & 26.7  & 64.5 & 92.4 & 71.7 \\
    \midrule
    \rowcolor{lightpink}
    Mimetic & 47.7  & 71.6 & 93.6 & 75.0 \\
    \midrule
    \rowcolor{lightgreen}
    Conditioned (ours) & 47.7 & 72.1  & 94.1  & 75.3  \\
    \bottomrule
\end{tabular}
\end{table*}

\begin{figure}[ht!]
    \centering
    \includegraphics[width=0.32\linewidth]{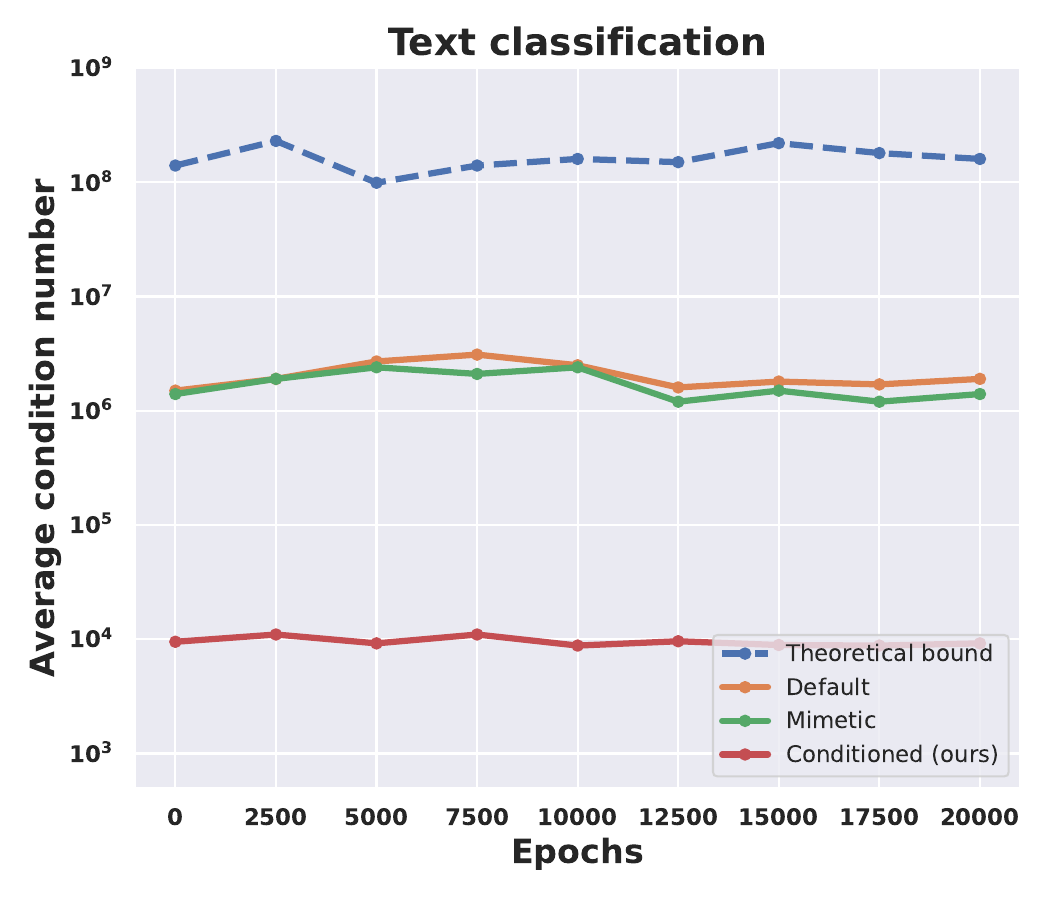}
    \hfill
    \includegraphics[width=0.32\linewidth]{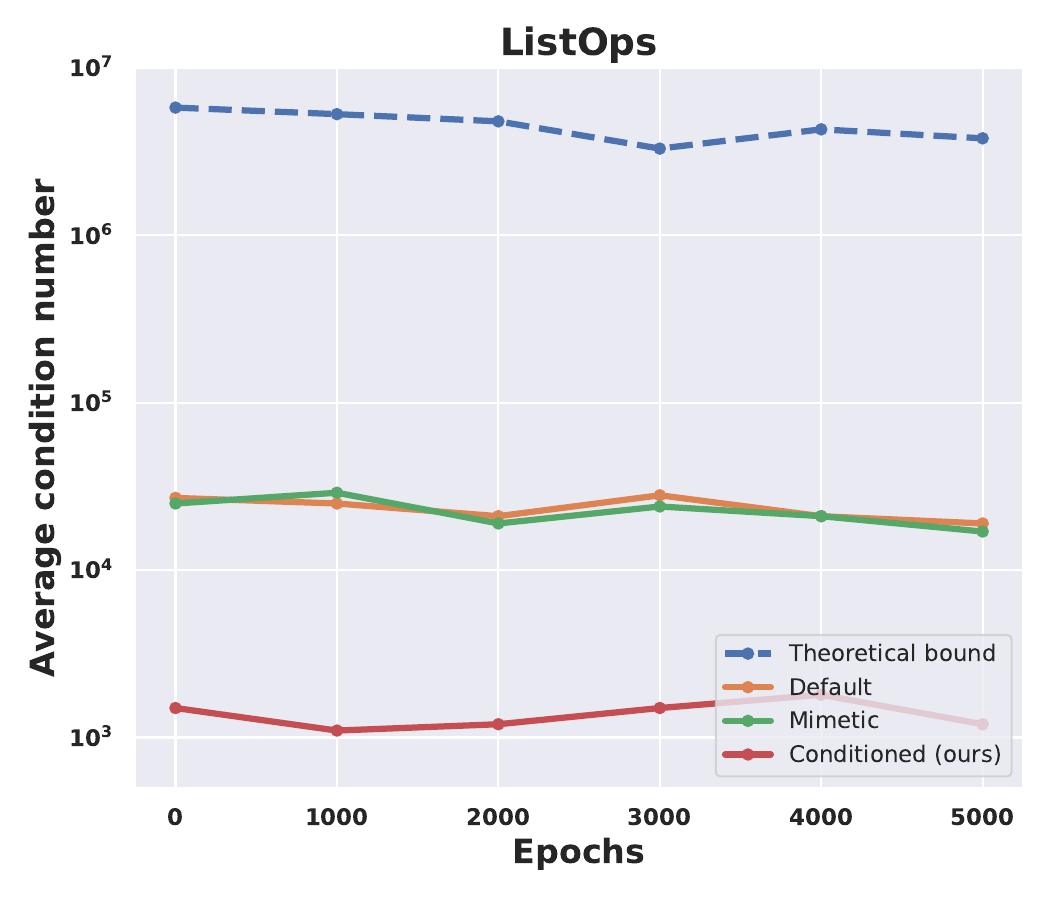}
    \hfill
    \includegraphics[width=0.32\linewidth, height=0.16\textheight]{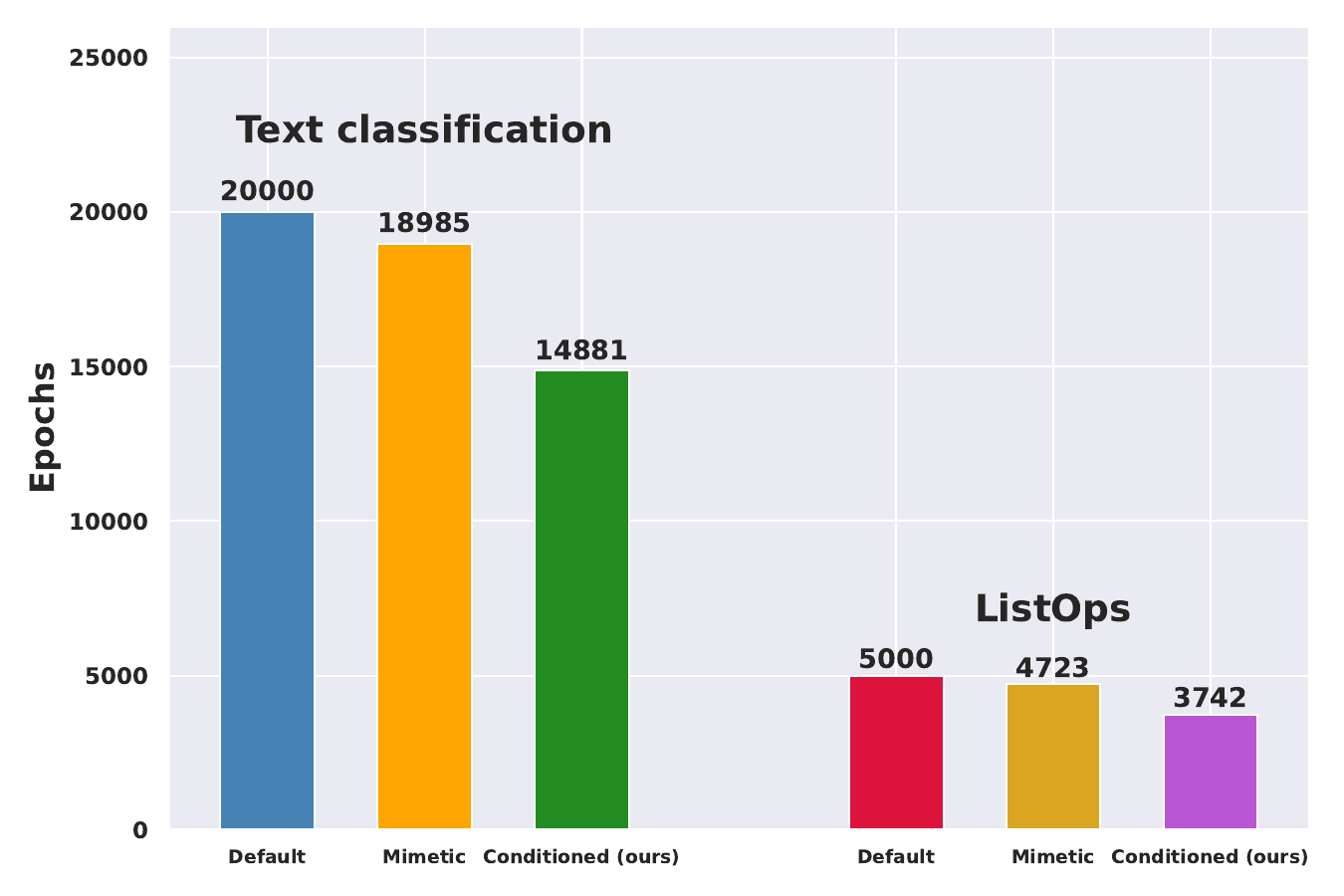}
    \caption{Average condition number of the attention Jacobian during training under three common initialization schemes, shown alongside the theoretical bound from \cref{eqn:cond_jac}.}
    \label{fig:nyst_analysis}
\end{figure}

\subsection{Object Detection and Instance Segmentation}\label{subsec:OD}

In this section, we test conditioned initialization in a fine-tuning setting on two downstream tasks: object detection and instance segmentation. We first pretrain an XCiT architecture~\citep{ali2021xcit} on ImageNet-1K and then fine-tune on COCO 2017~\citep{lin2014microsoft}. The XCiT models are used as backbones within a Mask R-CNN framework~\citep{he2017mask} equipped with a Feature Pyramid Network (FPN) for multi-scale features. To connect XCiT to FPN, we adapt its column structure to output intermediate representations, using the 12-layer XCiT-Small (XCiT-S) with strides adjusted from 16 to $[4, 8, 16, 32]$ for FPN compatibility. Downsampling is done with max pooling and upsampling with a single transposed convolution. We evaluate this setup across three initializations: default truncated normal, mimetic, and our conditioned initialization.

\paragraph{Results.} The results for object detection and instance segmentation are presented in \cref{tab:transferlearning}. 
We report \(AP^b\) (Average Precision for bounding boxes), \(AP^b_{50/75}\) (Average Precision at IoU thresholds of 0.50 and 0.75 for bounding boxes), \(AP^m\) (Average Precision for masks), and \(AP^m_{50/75}\) (Average Precision at IoU thresholds of 0.50 and 0.75 for masks). 
Across all metrics, XCiT models with conditioned initialization consistently outperform both alternatives.

\begin{table*}[!ht]
\caption{Performance evaluation of object detection and instance segmentation on the COCO dataset. For each metric, our spectrally conditioned architecture (Spec. cond.) outperforms the original.}
\label{tab:transferlearning}
\centering
\setlength{\tabcolsep}{7pt}

\begin{tabular}{cc c c c c c}
    \toprule
    \rowcolor{gray!10}
    Model & $AP^b$ & $AP^b_{50}$ & $AP^b_{75}$ & $AP^m$ & $AP^m_{50}$ & $AP^m_{75}$ \\
    \midrule
    Default & 44.9  & 66.1  & 48.9  & 40.1  & 63.1  & 42.8 \\
    \midrule
     Mimetic & 44.8  & 66.0  & 49.1  & 40.2  & 63.1  & 42.9 \\
    \midrule
    \rowcolor{lightgreen}
    Conditioned (ours) & 45.5  & 66.8  & 49.5  & 40.6  & 63.5  & 43.3  \\
    \bottomrule
\end{tabular}
\end{table*}

\subsection{Long Range Sequence Modeling}\label{subsec:nyst}

Long-range sequences are essential for Transformers, enabling the integration of information across distant tokens. We assess our initialization scheme on the Long-Range Arena (LRA) benchmark~\citep{tay2020long}, designed to evaluate models on extended inputs. For this, we use the Nystr\"{o}mformer~\citep{xiong2021nystromformer}, which achieves efficient long-range modeling via near-linear attention. We train three variants: one with default truncated normal initialization~\citep{xiong2021nystromformer,hf_nystromformer_doc}, one with mimetic initialization~\citep{trockman2023mimetic}, and one with our conditioned initialization from \cref{sec:spec_cond_init} following the setup of~\cite{xiong2021nystromformer}.

\paragraph{Results.} 
From \cref{tab:nystrom} we see that
across all tasks in the LRA benchmark, conditioned initialization consistently outperforms both default and mimetic initialization. We validate the theoretical results of \cref{sec:theory} on the text classification and ListOps task in the LRA benchmark suite. 
\Cref{fig:nyst_analysis} reports the average condition number of the Jacobian of the attention matrices 
for a Nystr\"{o}mformer on the text classification task (left) and a Nystr\"{o}mformer on the ListOps task (middle), during training, under the above mentioned three initializations, alongside the theoretical upper bound from \cref{thm:cond_jac}. 
The results demonstrate that conditioned initialization consistently yields a better-conditioned Jacobian, 
providing empirical support for its role in enabling more stable attention mechanisms. Furthermore, we measured the number of epochs required by mimetic and conditioned initialization to match the final accuracy attained by the default initialization.
\Cref{fig:nyst_analysis} (right) reports these results demonstrating that conditioned initialization consistently converges approximately 25\% faster to the same accuracy level.

\begin{table*}[!ht]
\caption{Nystr\"{o}mformer with three different initializations on the LRA benchmark. We report evaluation accuracy (\%). As shown, our initialization improves performance across all tasks.}
\label{tab:nystrom}
\centering
\setlength{\tabcolsep}{7pt}

\begin{tabular}{cc c c c c}
    \toprule
    \rowcolor{gray!10}
    Model & ListOps & Text & Retrieval & Image & Pathfinder \\
    \midrule
    Default & 37.1  & 63.8 & 79.8 & 39.9 & 72.9 \\
    \midrule
    \rowcolor{lightpink}
    Mimetic & 37.1  & 64.0 & 79.9 & 40.2 & 73.2 \\
    \midrule
    \rowcolor{lightgreen}
    Conditioned (ours) & 37.9 & 64.9 & 80.8 & 40.4  & 73.9 \\
    \bottomrule
\end{tabular}
\end{table*}

\subsection{Language Modeling}\label{sec:lang}

%\subsubsection{Language Modeling with Crammed BERT}

%We apply the insights from \cref{sec:theory} to the Crammed BERT language model~\citep{geiping2023cramming}, trained from scratch with masked language modeling. 
%We train three variants: the original Crammed BERT~\citep{geiping2023cramming} with the default normal initialization (\(\mu=0, \sigma=0.02\)), one with mimetic initialization, and one with our conditioned initialization. 
%All models are pretrained on The Pile~\citep{gao2021pile}, a large-scale corpus for language modeling, following the training setup of~\cite{geiping2023cramming}. 
%For evaluation, we use the GLUE benchmark~\citep{wang2018glue} under the same protocol as~\cite{geiping2023cramming}. 
%\Cref{tab:bert_glue_results} reports the results, showing that conditioned initialization achieves the best performance, with mimetic also outperforming the default initialization. Results for a GPT-2 architecture can be found in \cref{app:language}.

We apply the insights from \cref{sec:theory} to the Crammed BERT language model~\citep{geiping2023cramming}, trained with masked language modeling. We consider three variants: the original model with default normal initialization (\(\mu=0, \sigma=0.02\)), one with mimetic initialization, and one with our conditioned initialization. All models are pretrained on The Pile~\citep{gao2021pile} following the setup of~\cite{geiping2023cramming}, and evaluated on the GLUE benchmark~\citep{wang2018glue}. As shown in \cref{tab:bert_glue_results}, conditioned initialization yields the strongest performance, with mimetic also outperforming the default baseline. Additional results for GPT-2 are provided in \cref{app:language}.

\begin{table}[h]
%\scriptsize
\setlength{\tabcolsep}{7pt}  % Wider column padding
\caption{We evaluate a pretrained Crammed BERT with three different initialization schemes on the GLUE benchmark, and find that conditioned initialization outperforms the other two.}
\label{tab:bert_glue_results}
\centering
\begin{tabular}{lccccccccc}
    \toprule
    \rowcolor{gray!10}
    & MNLI & SST-2 & STSB & RTE & QNLI & QQP & MRPC & CoLA & Avg. \\
    \midrule
    Default & 83.8 & 92.3 & 86.3 & 55.1 & 90.1 & 87.3 & 85.0 & 48.9 & 78.6  \\
      \midrule
      \rowcolor{lightpink}
    Mimetic & 84.1 & 92.5 & 86.5 & 55.1 & 90.3 & 87.5 & 85.1 & 50.1 & 78.9  \\
    \midrule
    \rowcolor{lightgreen}
    Conditioned & 84.8 & 92.9 & 86.9 & 55.5 & 91.1 & 87.7 & 86.0 & 51.7 & 79.6 \\
    \bottomrule
\end{tabular}
\end{table}

%\vspace{-0.2cm}

\section{Limitations}\label{sec:lims}
Our conditioned initialization is derived by optimizing an upper bound on the condition number of the self-attention Jacobian, rather than minimizing the Jacobian’s condition number directly. 
The motivation was to examine whether such a bound-based initialization could induce a more favorable optimization bias for training Transformers. 
While this approach offers useful theoretical guidance and is consistent with the empirical gains we observe, it remains an indirect proxy. 
Developing methods that can efficiently estimate and control the exact Jacobian conditioning during training would therefore be a valuable direction for future work. 

%In addition, our experiments were limited to models with up to 100 million parameters as we did not have access to large-scale GPU resources required for training billion-parameter models. 
%Assessing whether conditioned initialization confers similar benefits at that scale is an important avenue for future research.

\section{Conclusion}\label{sec:conclusion}
In this paper, we introduced a theoretical framework that relates the conditioning of self-attention Jacobians to the spectral properties of the query, key, and value matrices in Transformer architectures. 
Building on this insight, we proposed a simple initialization strategy, conditioned initialization, that aims to reduce the condition number of the attention Jacobian at initialization, thereby providing a more favorable inductive bias for optimization. 
Extensive experiments show that this approach consistently improves performance across a wide range of Transformer models and tasks, including image classification, object detection, language modeling, and long-range sequence learning.
\footnote{Digital tools were used for grammar and formatting only. No large language models contributed to the research, and all findings are original work by the authors.}

\section*{Acknowledgments}
Both authors acknowledge support from Commonwealth Bank of Australia through the CommBank Centre for Foundational AI Research. This funding was essential for the completion of the research described in this publication.

%\subsubsection*{Acknowledgments}
%Use unnumbered third level headings for the acknowledgments. All
%acknowledgments, including those to funding agencies, go at the end of the paper.

\bibliography{iclr2026_conference}
\bibliographystyle{iclr2026_conference}

\newpage
\appendix
\section{Appendix}\label{sec:app}

\section*{Ethics Statement}
This work relies solely on publicly available datasets and does not involve human subjects, personally identifiable information, or sensitive data. The proposed methods are developed exclusively to advance fundamental research in machine learning.

\section*{Reproducibility Statement}
All experiments in this work were designed with reproducibility in mind. References are provided for any external codebases employed, and full details of training protocols and hardware are described in the appendix. Complete proofs of all theoretical results are also included to allow independent verification.  

\section*{Use of LLMs}
This manuscript was prepared with the assistance of LLMs for grammar checking only. No language models were used in conducting the research or drafting the scientific content.

\subsection{Theoretical Analysis}\label{app:theory}

In this section, we give the proofs of the lemma and theorems from \cref{sec:main_theory} and the proposition from \cref{sec:spec_cond_init}.

\paragraph{Notation.} For the convenience of the reader we restate the notation we used in \cref{sec:theory}.

Given a matrix $Z \in \R^{m\times n}$ we denote the vectorization of $Z$ by $\mathrm{vec}(Z) \in \R^{mn\times 1}$ \citep{magnus2019matrix}. Note that for such a matrix
there is a transformation $T_{mn} \in \R^{mn \times mn}$ such that
$T_{mn}\mathrm{vec}(Z) = \mathrm{vec}(Z^T)$ where $Z^T$ denotes the transpose of $Z$. The matrix $T_{mn}$ is known as a commutation matrix and is a permutation matrix \citep{magnus2019matrix}. 
The maximum singular value of a matrix $Z$ will be denoted by $\sigma_{\max}(Z)$ and the minimum singular value by $\sigma_{\min}(Z)$. We will use the standard terminology SVD to denote the singular value decomposition of a matrix. Given a vector $v \in \R^n$ the notation $||v||_2$ will denote the vector 2-norm of $v$. Finally, we will let $I_{m \times n}$ denote the rectangular identity matrix that has all $1$'s on its main diagonal. When we are dealing with square matrices we will often simply write $I_n$ with the understanding that $I_n$ is the $n \times n$ square identity matrix. Context will make it clear whether we are in the square or non-square regime. We also let
$\mathcal{O}_{m \times n}$ as the real $m \times n$ semi-orthogonal matrices.

We start with the following standard facts on derivatives of matrices. We point out to the reader that \cite{qi2025taming} also computes various Jacobians of the self-attention layer using Kronecker factorizations however their computations are slightly different to ours and we therefore give explicit details of how to obtain the proofs of \cref{lem:softmax_deriv} and \cref{thm:attn_derivs}.

\begin{lemma}\label{lem:matrix_deriv}
Let $A \in R^{n\times m}$, $B \in \R^{k\times l}$ and 
$C \in \R^{m\times k}$. Then 
\begin{equation}
    \frac{\partial{ACB}}{\partial{C}} = B^T \otimes A.
\end{equation}
\end{lemma}
\begin{proof}
We start by using a well known vectorization identity \citep{magnus2019matrix}
\begin{equation}\label{eqn:vec_identity}
    \mathrm{vec}(ACB) = (B^T \otimes A)\mathrm{vec}(C)
\end{equation}
where $\mathrm{vec}$ denotes the vectorization operator which takes a matrix and maps it to a vector by stacking its columns on top of each other, see \cite{magnus2019matrix}. We then differentiate \cref{eqn:vec_identity} to obtain
\begin{equation}
\frac{\partial{\mathrm{vec}(ACB)}}{\partial{\mathrm{vec}(C)}} = 
B^T \otimes A.
\end{equation}
The result of the lemma follows.
\end{proof}

\begin{lemma}\label{lem:transpose_deriv}
Let $A \in \R^{n\times m}$ so that $A^T \in \R^{m\times n}$. Then
\begin{align}
    \mathrm{vec}(A^T) &= T_{mn}\mathrm{vec}(A) \\
    \frac{\partial{\mathrm{vec}(A^T)}}{\partial{\mathrm{vec}(A)}} &= T_{mn}
\end{align}
where $T_{mn}$ is a commutation matrix.
\end{lemma}
\begin{proof}
The first equation follows from the definition of the transpose of a matrix \citep{magnus2019matrix}. The second equation then follows from the first.
\end{proof}

The proof of \cref{lem:softmax_deriv} is given as follows.

\begin{proof}[Proof of \cref{lem:softmax_deriv}]
Let $[z_1,\ldots ,z_n]$ denote a row vector in $\R^n$. Then by definition 
\begin{equation}
\mathrm{softmax}([z_1,\ldots ,z_n]) = 
\bigg{[}\frac{e^{z_1}}{\sum_{i=1}^ne^{z_i}},\ldots ,\frac{e^{z_n}}{\sum_{i=1}^ne^{z_i}}\bigg{]}. 
\end{equation}
From the above equation we can compute the partial derivative and find
\begin{equation}
\frac{\partial \, \mathrm{softmax}(z)_i}{\partial z_j} = \mathrm{softmax}(z)_i \left( \delta_{ij} - \mathrm{softmax}(z)_j \right).
\end{equation}
The term $\frac{\partial \, \mathrm{softmax}(z)_i}{\partial z_j}$ is precisely the $ij$ component of the matrix 
$\frac{\partial \mathrm{softmax}(z)}{\partial z}$. Putting each of these $ij$ terms into an $n \times n$ matrix we find
\begin{equation}
\frac{\partial \mathrm{softmax}(z)}{\partial z} = 
Diag(z) - z\cdot z^T
\end{equation}
which proves the lemma.
\end{proof}

We can use the above lemmas to give the proof of \cref{thm:attn_derivs}.

\begin{proof}[Proof of \cref{thm:attn_derivs}]
We start by establishing the derivative formula for the term 
$\frac{\partial\mathrm{A}(X)}{\partial W_Q}$. We will use the notation used in \cref{sec:prelims}. Note that by definition
$\frac{\partial\mathrm{A}(X)}{\partial W_Q} \in \R^{dN\times dD}$.
Using \cref{eqn:attn_eqn_general}
\begin{equation}\label{app:eqn:attn_eqn_general}
    \mathrm{A}(X) = I_N\mathrm{softmax}(XW_QW_K^TX^T)XW_V
\end{equation}
where $I_N$ is the $N \times N$ identity matrix. This is done so that we can apply \cref{lem:matrix_deriv}. We then compute
\begin{align}
\frac{\partial\mathrm{A}(X)}{\partial W_Q} &= 
\frac{\partial (I_N\mathrm{softmax}(XW_QW_K^TX^T)XW_V)}{\partial W_Q} \\
&= (W_V^TX^T\otimes I_N)\frac{\partial \mathrm{softmax}(XW_QW_K^TX^T)}{\partial W_Q} \text{ using } \cref{lem:matrix_deriv} \\
&=  (W_V^TX^T\otimes I_N) \Lambda(\mathrm{softmax}(XW_QW_K^TX^T))\frac{\partial (XW_QW_K^TX^T)}{\partial W_Q} 
\text{ using } \cref{lem:softmax_deriv} \\
&= (W_V^TX^T\otimes I_N) \Lambda(\mathrm{softmax}(XW_QW_K^TX^T))
(XW_K\otimes X)
\end{align}
which proves the firs equality in \cref{thm:attn_derivs}. 

To compute $\frac{\partial \mathrm{A}(X)}{\partial W_K} \in \R^{dN\times dD}$ we proceed in a similar way. 
\begin{align}
\frac{\partial \mathrm{A}(X)}{\partial W_K} &= 
\frac{\partial (I_N\mathrm{softmax}(XW_QW_K^TX^T)XW_V)}{\partial W_K} \\
&= 
(W_V^TX^T\otimes I_N)\frac{\partial \mathrm{softmax}(XW_QW_K^TX^T)}{\partial W_K} \text{ using } \cref{lem:matrix_deriv} \\
&=
(W_V^TX^T\otimes I_N) \Lambda(\mathrm{softmax}(XW_QW_K^TX^T))\frac{\partial (XW_QW_K^TX^T)}{\partial W_K} 
\text{ using } \cref{lem:softmax_deriv} \\
&=
(W_V^TX^T\otimes I_N) \Lambda(\mathrm{softmax}(XW_QW_K^TX^T))
(X\otimes XW_Q)T_{Dd} 
\end{align}
where the last equality follows from \cref{lem:matrix_deriv,lem:transpose_deriv}
This establishes the second equality in \cref{thm:attn_derivs}.

To prove the identity for $\frac{\partial\mathrm{A}(X)}{\partial W_V} \in \R^{dN\times dD}$ we write
\begin{equation}
\mathrm{A}(X) = \mathrm{softmax}(XW_QW_K^TX^T)XW_VI_d    
\end{equation}
where $I_d$ is the $d \times d$ identity matrix. Then we
simply apply \cref{lem:matrix_deriv} to obtain
\begin{align}
\frac{\partial\mathrm{A}(X)}{\partial W_V} &= 
\frac{\partial (\mathrm{softmax}(XW_QW_K^TX^T)XW_VI_d)}{\partial W_V} \\
&= I_d \otimes \mathrm{softmax}(XW_QW_K^TX^T)X
\end{align}
which proves the final equality in \cref{thm:attn_derivs}. 
\end{proof}

We also give the proof of \cref{thm:cond_jac}.

\begin{proof}[Proof of \cref{thm:cond_jac}]
The proof of \cref{thm:cond_jac} follows from using \cref{thm:attn_derivs} and the definition of the Jacobian of $\mathrm{A}(X)$ with respect to $W_Q$, $W_K$ and $W_V$ given by
\begin{equation}
 J(\mathrm{A}(X)) =  
 \bigg{[}\frac{\partial{(\mathrm{A}(X))}}{\partial{W_Q}}, 
 \frac{\partial{(\mathrm{A}(X))}}{\partial{W_K}}, 
 \frac{\partial{(\mathrm{A}(X))}}{\partial{W_V}}\bigg{]}^T.  
\end{equation}
We recall that the condition number is defined as 
$\kappa(J(\mathrm{A}(X))) = \frac{\sigma_{\max}(J(\mathrm{A}(X)))}{\sigma_{\min}(J(\mathrm{A}(X)))}$ where 
$\sigma_{\max}(J(\mathrm{A}(X)))$ is the maximum singular value of $J(\mathrm{A}(X))$ and 
$\sigma_{\min}(J(\mathrm{A}(X)))$ the minimum singular value which we know is non-zero because of the assumption that $J(\mathrm{A}(X))$ has full rank.
Note that, using the notation in \cref{sec:prelims}, we have that
$J(\mathrm{A}(X)) \in \R^{3dN \times dD}$ as 
$\frac{\partial\mathrm{A}(X)}{\partial W_Q}$, $\frac{\partial \mathrm{A}(X)}{\partial W_K}$, $\frac{\partial\mathrm{A}(X)}{\partial W_V} \in \R^{dN \times dD}$. For each of notation we will write 
$\mathrm{A}_Q := \frac{\partial\mathrm{A}(X)}{\partial W_Q}$, 
$\mathrm{A}_K := \frac{\partial \mathrm{A}(X)}{\partial W_K}$, 
$\mathrm{A}_V := \frac{\partial\mathrm{A}(X)}{\partial W_V}$.

We will start by computing a bound for the maximum singular value. We have for any vector $z \in \R^{dD}$ we have
\begin{align}
||J(\mathrm{A}(X))z||_2^2 &= 
 \left\lVert\bigg{[}\frac{\partial{(\mathrm{A}(X))}}{\partial{W_Q}}(z), 
 \frac{\partial{(\mathrm{A}(X))}}{\partial{W_K}}(z), 
 \frac{\partial{(\mathrm{A}(X))}}{\partial{W_V}}(z)\bigg{]}^T\right\rVert_2^2 \\
&=  \left\lVert\big{[}\mathrm{A}_Q(z), 
 \mathrm{A}_K(z), 
 \mathrm{A}_V(z)\big{]}\right\rVert_2^2 \\
&= 
\left\lVert \mathrm{A}_Q(z)\right\rVert_2^2 + 
\left\lVert \mathrm{A}_K(z)\right\rVert_2^2 +
\left\lVert \mathrm{A}_V(z)\right\rVert_2^2 \\
&\leq 
\left(\sigma_{\max}(\mathrm{A}_Q)^2 + \sigma_{\max}(\mathrm{A}_K)^2 + 
\sigma_{\max}(\mathrm{A}_V)^2\right)||z||_2^2. 
\end{align}
This implies that
\begin{align}
\sigma_{\max}(J(\mathrm{A}(X))) &:= \max_{z\neq 0}\frac{||J(\mathrm{A}(X))z||_2^2}{||z||_2^2} \\   
&\leq
\sqrt{\sigma_{\max}(\mathrm{A}_Q)^2 + \sigma_{\max}(\mathrm{A}_K)^2 + 
\sigma_{\max}(\mathrm{A}_V)^2} \\
&\leq \sigma_{\max}(\mathrm{A}_Q) + \sigma_{\max}(\mathrm{A}_K) + 
\sigma_{\max}(\mathrm{A}_V).
\end{align}
The next step is to compute a lower bound for the minimum singular value $\sigma_{\min}(J(\mathrm{A}(X)))$. The approach is similar to the above, using the fact that 
$\sigma_{\min}(J(\mathrm{A}(X))) = 
\min_{||z||_2 = 1}J(\mathrm{A}(X)(z)$. We can then use the inequality
\begin{align}
\left\lVert J(\mathrm{A}(X)(z)\right\rVert_2^2 &= 
\left\lVert
\bigg{[}\frac{\partial{(\mathrm{A}(X))}}{\partial{W_Q}}(z), 
 \frac{\partial{(\mathrm{A}(X))}}{\partial{W_K}}(z), 
 \frac{\partial{(\mathrm{A}(X))}}{\partial{W_V}}(z)\bigg{]}^T\right\rVert_2^2 \\
&\geq
\max\left\{\left\lVert\frac{\partial{(\mathrm{A}(X))}}{\partial{W_Q}}(z) \right\rVert,
\left\lVert\frac{\partial{(\mathrm{A}(X))}}{\partial{W_K}}(z) \right\rVert,
\left\lVert\frac{\partial{(\mathrm{A}(X))}}{\partial{W_V}}(z) \right\rVert
\right\} 
\end{align}
Then minimizing the above over the constraint $z \in \R^{dD}$ such that $||z|| = 1$ we obtain
\begin{align}
\sigma_{\min}(J(\mathrm{A}(X))) \geq 
\max\left\{
\sigma_{\min}\left(\frac{\partial{(\mathrm{A}(X))}}{\partial{W_Q}}\right), 
\sigma_{\min}\left(\frac{\partial{(\mathrm{A}(X))}}{\partial{W_K}}\right), 
\sigma_{\min}\left(\frac{\partial{(\mathrm{A}(X))}}{\partial{W_V}}\right)
\right\}.
\end{align}
Combing the bounds on $\sigma_{\max}(J(\mathrm{A}(X)))$ and 
$\sigma_{\min}(J(\mathrm{A}(X)))$ we obtain
\begin{align}
\kappa(J(\mathrm{A}(X))) \leq \kappa\left(\frac{\partial{(\mathrm{A}(X))}}{\partial{W_Q}}\right) + 
\kappa\left(\frac{\partial{(\mathrm{A}(X))}}{\partial{W_K}}\right)
+
\kappa\left(\frac{\partial{(\mathrm{A}(X))}}{\partial{W_V}}\right).
\end{align}
The final step is to get a bound on the condition numbers for each term on the right hand side in the above inequality. This is done using two facts: Firstly, given two matrices $C$ and $D$ such that 
the product $CD$ is full rank then $\kappa(CD) \leq \kappa(C)\kappa(D)$ and the second that $\kappa(C \otimes D) = 
\kappa(C)\kappa(D)$. Using these two facts we can then use 
\cref{thm:attn_derivs} to obtain the bound of \cref{thm:cond_jac} and the proof is finished.
\end{proof}

Using \cref{thm:cond_jac} the proof of proposition is straightforward.

\begin{proof}[Proof of \cref{prop:u_bound_init}.]
We first observe that by definition of the condition number of a general $m \times n$ matrix $M$ must always satisfy $\kappa(M) \geq 1$. For matrices with $1$'s on the diagonal and zero elsewhere the condition number is $1$. If $M$ is a semi-orthogonal matrix then either
\begin{align}
    &1. \text{ } MM^T = I_{m\times m} \text{ if } m \geq n \\
    &2. \text{ } M^TM = I_{n\times n} \text{ if } n \geq m.
\end{align}
Since the singular values of $M$ are precisely the eigenvalues of $MM^T$ if $m \geq n$ or $M^TM$ if $n \geq m$. It follows that the singular values of $M$ must all be $1$ and hence it has condition number $1$.

Therefore, if the $W_Q$, $W_K$ and $W_V$ matrices are initialized as $I_{D\times d}$ or as a matrix in 
$\mathcal{O}_{D\times d}$ we have that their condition number is $1$. Therefore, we must have
\begin{equation}
    \mathcal{B}(\overline{\mathrm{A}}) \leq \mathcal{B}(\mathrm{A})
\end{equation}
and the proposition is proved.
\end{proof}

\begin{remark}
In the implementation strategy in \cref{sec:main_theory} we saw that we initialized the values matrix $W_V$ with a rectangular identity matrix $I_{D\times d}$. However, from the theory of that section we could have also initialized it with $\lambda I_{D\times d}$ for $\lambda \neq 0$. In general, we found empirically that this could be done but that if $\lambda$ got too large we noticed some instability in training due to $W_V$ having weights that were too large. Therefore, opting for the identity $I_{D\times d}$ was what we found worked well.
\end{remark}

\paragraph{Why Conditioning the Jacobian Aids Optimization.}
The rationale for improving the conditioning of the self-attention Jacobian is connected to established results on the Neural Tangent Kernel (NTK). Prior work, see \cite{liu2022loss}, has shown that better-conditioned NTKs lead to faster and more reliable convergence to a global minima during gradient-based optimization. Since the singular values of a network's Jacobian correspond to the positive square roots of the eigenvalues of its NTK, improving the conditioning of the Jacobian directly enhances the conditioning of the NTK. This connection provides a theoretical basis for why controlling the spectral structure of the self-attention Jacobian can benefit optimization. Furthermore, recent extensions of NTK theory to transformer architectures (e.g., Yang 2020) support the relevance of these insights in the attention setting. Our initialization scheme leverages this relationship by explicitly targeting improved conditioning at initialization, which is consistent with the empirical performance gains observed across tasks.

\subsubsection{Implementation details}\label{sec:imp}

In this section, we give the details of how we implement the semi-orthogonal initialization of the $W_Q$ and $W_V$ matrix of each head.

We recall from \cref{sec:spec_cond_init} that our initialization for the queries and keys, proceeded by initializing each $W_Q^{(i)}$ and $W_K^{(i)}$ in the $i$-th head with independent semi-orthogonal projections,
\[
(W_Q^{(i)})^T W_Q^{(i)} = I_d, 
\qquad 
(W_K^{(i)})^T W_K^{(i)} = I_d,
\]
for $i=1,\ldots,h$, where $h$ is the number of heads. To do this suppose each $W_Q^{(i)}$ and $W_K^{(i)}$ are $D \times d$ and let $r = \min(D, d)$ for each head we form two random matrices $R^Q_i$ and $R^K_i$ of shape $D \times d$ and  then take the truncated SVD
\begin{equation}
  U^Q_i(r)S^Q_i(r)(V^Q_i)^T(r) \text{ and }
    U^K_i(r)S^K_i(r)(V^K_i)^T(r)  
\end{equation}
where each of the  $U^Q_i(r) \in \sR^{D \times r}$ 
and $(V^Q_i)(r) \in \sR^{r \times d}$ and similarly for the $K$ ones. We then observe that 
\begin{equation}
    O_i^{Q} := U^Q_i(r)\cdot(V^Q_i)^T(r) \text{ and }
    O_i^{K} := U^K_i(r)\cdot(V^K_i)^T(r) 
\end{equation}
are semi-orthogonal. Doing this for each different head we get different semi-orthogonal matrices for each $W_Q$ and $W_K$ for each head. The implementation of $W_V$ is the same for each head and is simply done by fixing $W_V = I_{D\times d}$.

\begin{remark}
We note that in the above we used the SVD to obtain the initializations of each $W_Q$ and $W_V$. One can also use the QR decomposition \citep{magnus2019matrix} and PyTorch has a built in way to do this via \texttt{nn.init.orthogonal}.
\end{remark}

\subsection{Discussion on Normalization}
This section provides additional clarification on how normalization and stabilization mechanisms interact with the Jacobian conditioning analysis and with the proposed conditioned initialization scheme. Several modern transformer architectures apply normalization directly to the queries, keys, or values (e.g., RMSNorm, QKNorm), while Layer Normalization (LN) is typically applied to the inputs of each block. The discussion below outlines how these components relate to the theoretical results in \cref{sec:theory} and to the practical behaviour observed in \cref{sec:exps}.

\medskip
\noindent\textbf{Theoretical setting.}
In the original self-attention formulation \cite{vaswani2017attention}, no normalization is applied directly to the query, key, or value weight matrices. The only modification is the fixed scaling factor $1/\sqrt{d}$, determined by the head dimension. Since this factor does not depend on the model parameters, differentiating the attention map simply scales the Jacobian by a constant. A constant scalar multiplication does not change the condition number of a matrix; therefore, this scaling has no effect on the conditioning analysis developed in Section~3. The derivations for vanilla self-attention thus remain mathematically correct and serve as a foundation for the initialization scheme proposed later.

\medskip
\noindent\textbf{Compatibility with modern normalization layers.}
Many contemporary transformer variants incorporate normalization directly into the attention pathway, for example through RMSNorm or QKNorm applied to queries, keys, or values. The conditioned initialization introduced in this paper is designed to operate on top of these mechanisms rather than as a replacement for them. In all experiments, the architectures retain their original normalization layers exactly as implemented in the publicly released codebases. For example, the DeiT-B model applies QKNorm to both queries and keys. When applying conditioned initialization, these QKNorm layers remain in place. This ensures that comparisons in \cref{sec:exps} are consistent with the established baseline implementations in the literature.

\medskip
\noindent\textbf{Normalization versus conditioning.}
Normalization and conditioning address different aspects of stability. Normalization controls the scale of activations and gradients, whereas conditioning concerns the spectral structure of the Jacobian. These effects are distinct. A matrix may have small norm but poor conditioning, or it may have large norm yet be well-conditioned. For example,
\[
A = \begin{bmatrix} 1 & 0 \\[3pt] 0 & 1 \end{bmatrix}, \qquad
B = \begin{bmatrix} 10.1 & 0 \\[3pt] 0 & 10 \end{bmatrix}
\]
have very different Frobenius norms but nearly identical condition numbers, while
\[
C = \begin{bmatrix} 1 & 0 \\[3pt] 0 & 1 \end{bmatrix}, \qquad
D = \begin{bmatrix} \sqrt{2} & 0 \\[3pt] 0 & 0.1 \end{bmatrix}
\]
have similar Frobenius norms but condition numbers $1$ and approximately $14.14$, respectively. Normalization schemes such as LN, RMSNorm, and QKNorm regulate magnitudes, while conditioned initialization directly shapes the spectral behaviour of the Jacobian. These mechanisms therefore complement one another.

%\medskip
%\noindent\textbf{Additional empirical observations.}
%To further illustrate the interaction between initialization and normalization, we include two sets of experiments. The first compares several initialization strategies both with and without LN applied to the inputs. Removing LN makes all models significantly harder to train, yet conditioned initialization consistently yields the best performance among the methods tested. The second set evaluates the effect of RMSNorm and QKNorm within the attention mechanism. Across all architectures (ViT-B, DeiT-B, Swin-B, XCiT-M, DaViT-B) and across all normalization settings (none, RMSNorm, QKNorm), conditioned initialization improves performance relative to the baselines.

%These results indicate that the proposed initialization scheme is compatible with and complementary to existing normalization techniques, and that it provides consistent improvements regardless of whether LN, RMSNorm, or QKNorm is present in the architecture.

\paragraph{Ablation on Partial Initialization.}
We conducted several ablation studies in which the conditioned initialization was applied to only a subset of the attention projection matrices, such as initializing only $W_Q$, only $W_K$, or only $W_V$. An initial hypothesis was that initializing $W_V$ alone might be sufficient, since the softmax operation effectively performs a row-wise normalization on the attention scores. However, normalization and conditioning target fundamentally different properties: a matrix may have small norm yet still exhibit a large condition number. This observation led us to a more complete analysis, formalized in \cref{thm:cond_jac}, which shows that the conditioning of the self-attention Jacobian depends on the joint spectral structure of $W_Q$, $W_K$, and $W_V$. Consistent with this theoretical insight, we found that applying conditioned initialization to all three projections yields the most stable behaviour and the strongest empirical performance.

\paragraph{Discussion on the Output Projection.}
We also examined whether the conditioned initialization should be applied to the output projection matrix $W_O$, given its close relationship to $W_V$. Empirically, we found that once $W_V$ is initialized using the proposed scheme, the initialization of $W_O$ has negligible effect. Applying the conditioned initialization to $W_O$ in addition to $W_V$ did not alter training behaviour or final performance. Consequently, our method focuses on conditioning the query, key, and value projections, while leaving $W_O$ with its standard initialization.

\subsection{Experiments}\label{app:exps}

\subsubsection{Vision Transformers}\label{app:vision_trans}

\paragraph{Hardware and implementation.} The image classification experiments in \cref{subsec:IC} of the paper were done on Nvidia A100 GPUs.
The implementation of the ViTs was all done using the Timm code base \citep{rw2019timm}. The architectures were all trained from scratch on the ImageNet-1k dataset using the AdamW optimizer following the hyperparameters used in the original papers \citep{dosovitskiy2020image, steiner2106train, liu2021swin,ali2021xcit,touvron2021training,ding2022davit}. For the case of ViT-T on the Pets, Flowers, CIFAR-10 and CIFAR-100 datasets we used \cite{rw2019timm} for the architecture and the training hyperparameters from \cite{xu2023initializing}.

\paragraph{Optimization analysis for DaViT-B.} As shown in \cref{tab:vits}, conditioned initialization achieves higher accuracy under the same training regime compared to both default and mimetic initialization on the DaViT-B architecture.
To further assess training efficiency, we measured the number of epochs required by mimetic and conditioned initialization to match the final accuracy attained by the default initialization across on the DaViT-B architecture. 
\Cref{fig:davit_converge} reports these results, demonstrating that conditioned initialization consistently converges to approximately 25\% faster to the same accuracy level.

\begin{figure}[ht!]
    \centering
    \includegraphics[width=0.6\linewidth]{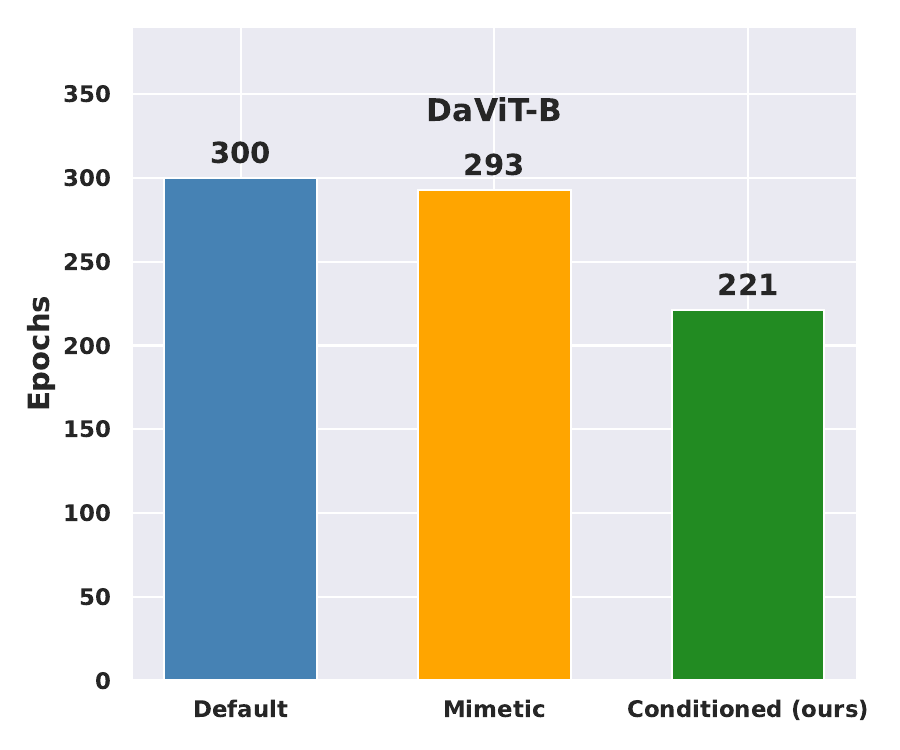}
    \caption{Total number of epochs required for each initialization to reach the final accuracy of the default initialization reported in \cref{tab:vits} for DaViT-B architecture. Conditioned initialization converges more quickly, requiring fewer epochs than both default and mimetic initialization.}
    \label{fig:davit_converge}
\end{figure}

\paragraph{Weight Selection} Xu et al.~\citep{xu2023initializing} introduced the weight selection method, showing that weights transferred from an ImageNet-21K-pretrained ViT-Small (ViT-S) model provide a strong inductive bias for initializing a ViT-Tiny (ViT-T) on small-scale datasets. 
Here, we examine whether weight selection remains effective when pretraining is performed on ImageNet-1K. 
To this end, we pretrained ViT-S models on ImageNet-1K under three initialization schemes: default truncated normal~\citep{timm2025}, mimetic~\citep{trockman2023mimetic}, and our conditioned initialization from \cref{sec:spec_cond_init}. 
Using the weight selection procedure~\citep{xu2023initializing}, we then derived corresponding ViT-T initializations and trained them on Food-101~\citep{bossard2014food}, CIFAR-10, and CIFAR-100~\citep{krizhevsky2009learning}. 
\Cref{tab:ws} summarizes the results. 
Each entry labeled ``ImageNet-1K" indicates that the ViT-S model was pretrained on ImageNet-1K with the specified initialization before its weights were transferred to ViT-T via weight selection. 
For comparison, the final row reports the performance of a ViT-T initialized from an ImageNet-21K-pretrained ViT-S with default truncated normal initialization. 
The results highlight a key finding: replacing ImageNet-21K with ImageNet-1K can yield comparable performance, provided that the initialization is chosen carefully. 
In particular, conditioned initialization on ImageNet-1K achieves accuracy on par with default initialization pretrained on ImageNet-21K, underscoring its effectiveness in data-limited pretraining regimes.

\begin{table*}[!ht]
\caption{Comparison of ViT-T pretrained on four small scale datasets with three different initializations. We report Top-1\% classification accuracy. In each case, conditioned initialization improves performance over the default and mimetic initializations. }
\label{tab:ws}
\centering
\begin{tabular}{c|c c c }
    \toprule
    \rowcolor{gray!10}
    & Food-101 & CIFAR-10 & CIFAR-100 \\
    \midrule
    ImageNet-1k + Default  & 85.5  & 96.6 & 79.7 \\
    \midrule
    %\rowcolor{lightpink}
    ImageNet-1k + Mimetic  & 86.4  & 96.3 & 79.9 \\
    \midrule
    %\rowcolor{lightgreen}
    ImageNet-1k + Conditioned (ours) & \textbf{87.3} & \textbf{97.1} & 81.0  \\
    \midrule
    \midrule
    %\rowcolor{lightgreen}
    ImageNet-21k + Default & 87.1 & \textbf{97.1}  & \textbf{81.1}  \\
    \bottomrule
\end{tabular}
\end{table*}

\paragraph{Training loss curves.} We plot the training curves for each of the vision transformer experiments with each different initialization. In \cref{fig:vit_deit_loss}, \cref{fig:swin_xcit_loss}, \cref{fig:davit_loss} we plot the training loss for the different initializations for the ViT-B, DeiT-B, Swin-B, XCiT-M and DaViT-B architectures respectively. In each case, we see conditioned initialization has trains stably and converges faster when compared to the default and Mimetic initializations.

\begin{figure}[ht!]
    \centering
    \includegraphics[width=0.48\linewidth, height=0.22\textheight]{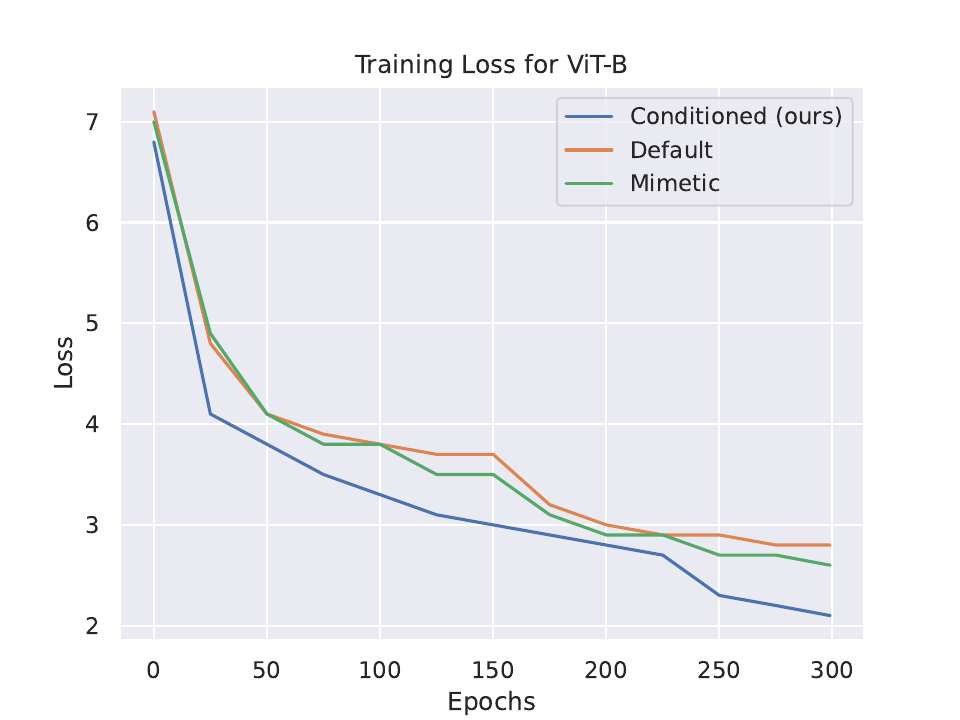}
    \hfill
    \includegraphics[width=0.48\linewidth, height=0.22\textheight]{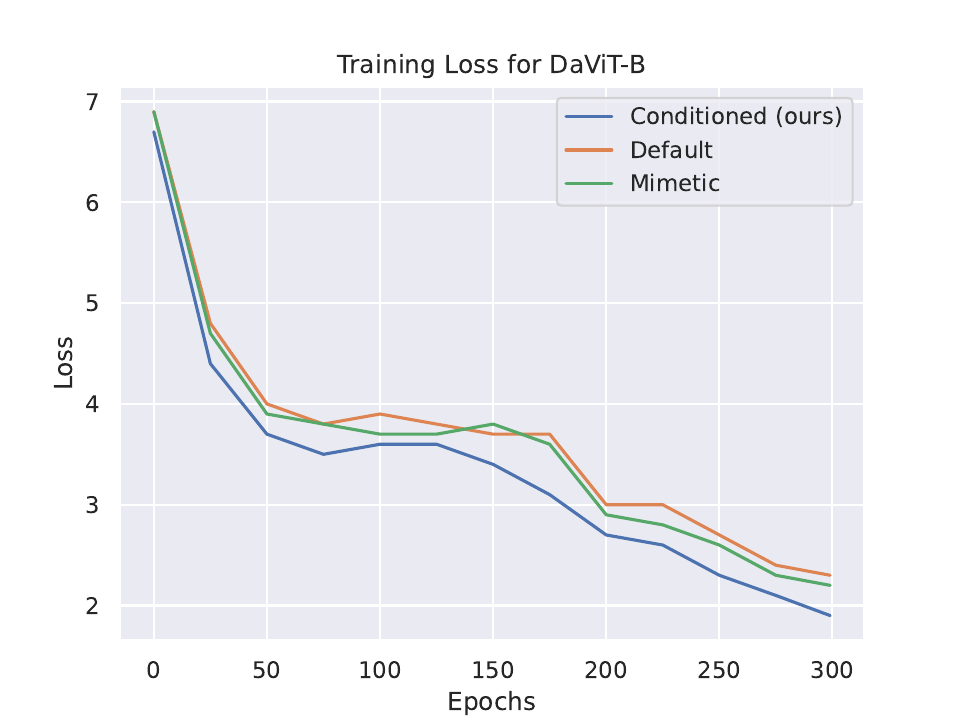}
    \caption{Training loss curves for different initializations for ViT-B (left) and DeiT-B (right).}
    \label{fig:vit_deit_loss}
\end{figure}

\begin{figure}[ht!]
    \centering
    \includegraphics[width=0.48\linewidth, height=0.22\textheight]{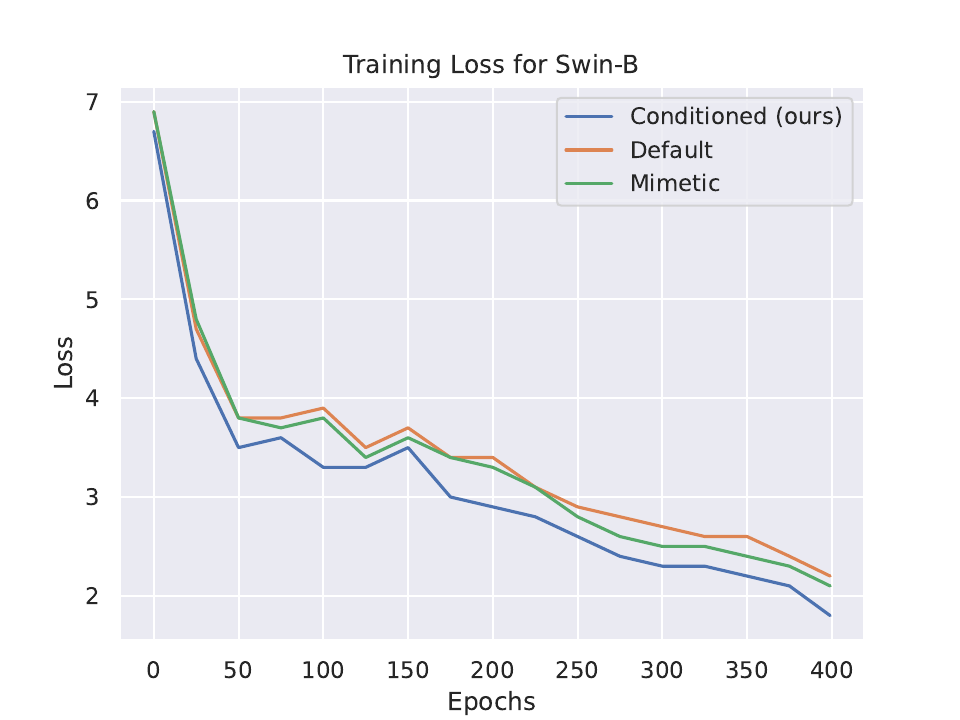}
    \hfill
    \includegraphics[width=0.48\linewidth, height=0.22\textheight]{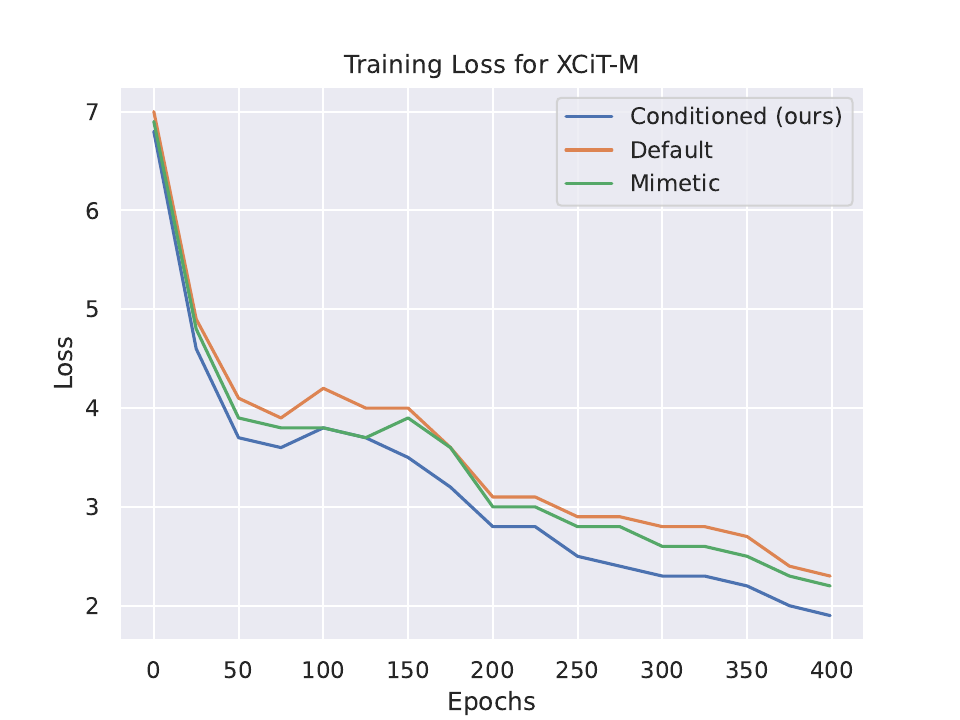}
    \caption{Training loss curves for different initializations for Swin-B (left) and XCiT-M (right).}
    \label{fig:swin_xcit_loss}
\end{figure}

\begin{figure}[ht!]
    \centering
    \includegraphics[width=0.48\linewidth, height=0.22\textheight]{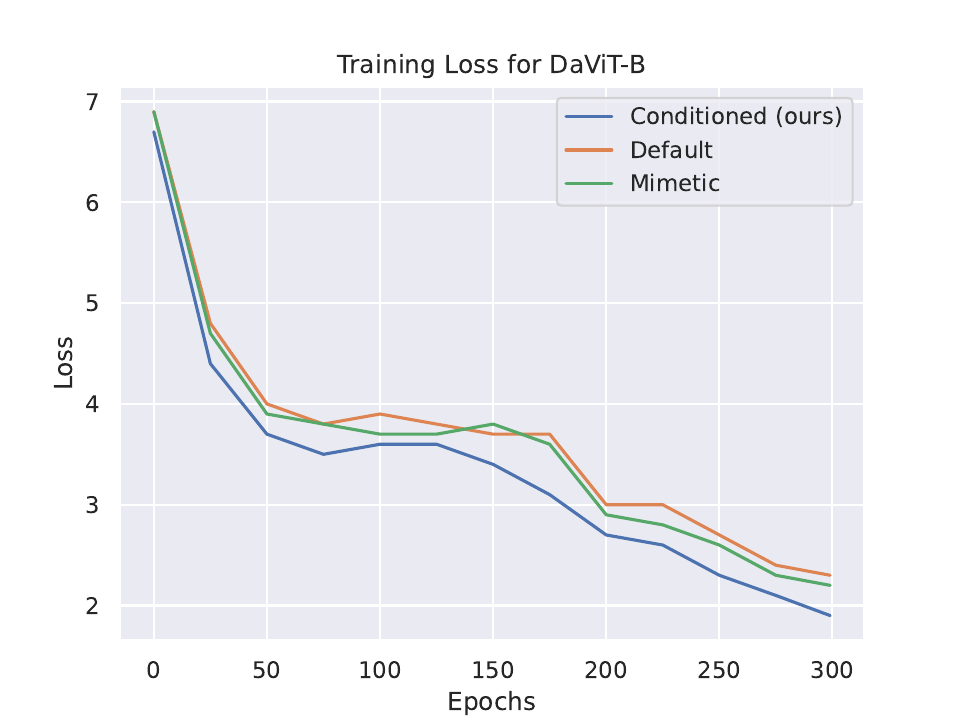}
    \caption{Training loss curves for different initializations for DaViT-B.}
    \label{fig:davit_loss}
\end{figure}

\subsubsection{Object detection and instance segmentation}\label{app:od}

\paragraph{Hardware and Implementation:} The experiments for \cref{subsec:OD} of the paper on object detection and instance segmentation were carried out on Nvidia A100 GPUs. The implementation followed 
 \cite{he2017mask}. We used the code base given by the GitHub \cite{matterport_maskrcnn} following their exact training regime.

 \subsubsection{Long Range Sequence Modeling with Nystr\"omformer}\label{app:nyst}

 \paragraph{Hardware and Implementation.} All the experiments for the
Nystr\"omformer on LRA benchmark results in \cref{subsec:nyst} were carried out on Nvidia A100 GPUs following the implementation and hyperparameter settings given in \cite{nystromformer_github}.

\subsubsection{Language Modeling}\label{app:language}

\paragraph{Crammed BERT.} The BERT language modeling experiment in \cref{sec:lang} were all carried out on a Nvidia A6000 GPU. The Crammed-Bert was implemented following the original paper \cite{geiping2023cramming} and the original GitHub \cite{CrammingGit}. The training regime follows \cite{CrammingGit}.

\paragraph{GPT-2 on TinyStories.} We train an autoregressive Transformer, namely a GPT-2 architecture
trained on the TinyStories dataset \citep{eldan2023tinystories}.
Once again we compared three initializations: one with the default normal initialization (\(\mu=0, \sigma=0.02\)), one with mimetic and a third with conditioned initialization. 
As shown in \cref{tab:gpt_results}, conditioned initialization achieves a lower perplexity than both the default and mimetic initializations showing that it boosts performance in the setting of autoregressive Transformers.
We used the training regime from \citep{eldan2023tinystories}.

\begin{table}[h]
    \centering
    %\small
    \setlength{\tabcolsep}{7pt}
    \caption{GPT-2 models trained on the TinyStories dataset. We initialize a model in three different ways. Conditioned initialization achieves a perplexity than the other two initializations.}
    \label{tab:gpt_results}
    \renewcommand{\arraystretch}{1.15}
    \begin{tabular}{lccc cc}
        \toprule
         & Perplexity \\ 
        \midrule
        Default & 2.47 \\
        \midrule
        Mimetic & 2.40 \\
        \midrule
        Conditioned (ours)  & 2.29\\
        \bottomrule
    \end{tabular}
\end{table}

\subsection{Larger Scale Experiments}\label{app:large_exps}

In this section we tested our initialization on much larger models in both the vision and language setting.

We ran a larger GPT-2 model, 472 million parameters, on the WikiText-103 dataset and the TinyStories dataset. The results are shown in \cref{tab:ts_400} and \cref{tab:wikitext_500} below clearly shows our initialization obtains better performance by obtaining the lowest perplexity.

\begin{table}[!ht]
    \centering
    \caption{Perplexity of 472M GPT-2 model pretrained on TinyStories.}
    \begin{tabular}{lc}
        \toprule
        Model & Perplexity \\
        \midrule
        Default & 2.39 \\
        Mimetic & 2.32 \\
        Conditioned (ours) & 2.20 \\
        \bottomrule
    \end{tabular}
    \label{tab:ts_400}
\end{table}

\begin{table}[!ht]
    \centering
    \caption{Perplexity of 472M GPT-2 model pretrained on WikiText-103.}
    \begin{tabular}{lc}
        \toprule
        Model & Perplexity \\
        \midrule
        Default & 44.2 \\
        Mimetic & 43.8 \\
        Conditioned (ours) & 42.7 \\
        \bottomrule
    \end{tabular}
    \label{tab:wikitext_500}
\end{table}

We repeated this experiment with a 1.72 billion parameter GPT-2 model. In this case we witnessed overfitting in all initialized cases. However, this is to be expected as a billion parameter model is too large for the TinyStories and WikiText-103 datasets. Yet even in this case our initialization performed much better as can be seen from \cref{tab:ts_4001b} and \cref{tab:wikitext_5001b}.

\begin{table}[!ht]
    \centering
    \caption{Perplexity of 1.72B GPT-2 model pretrained on TinyStories.}
    \begin{tabular}{lc}
        \toprule
        Model & Perplexity \\
        \midrule
        Default & 4.15 \\
        Mimetic & 4.25 \\
        Conditioned (ours) & 4.03 \\
        \bottomrule
    \end{tabular}
    \label{tab:ts_4001b}
\end{table}

\begin{table}[!ht]
    \centering
    \caption{Perplexity of 1.72B GPT-2 model pretrained on WikiText-103.}
    \begin{tabular}{lc}
        \toprule
        Model & Perplexity \\
        \midrule
        Default & 48.1 \\
        Mimetic & 48.5 \\
        Conditioned (ours) & 46.9 \\
        \bottomrule
    \end{tabular}
    \label{tab:wikitext_5001b}
\end{table}

We also ran experiments on large scale ViTs on the ImageNet-1k dataset, where these models range from 200-300M parameters. Once again for these larger models we witnessed overfitting (this has also been observed in \cite{dosovitskiy2020image, ding2022davit}) yet even in this case our initialization outperformed the other two standard ones, see \cref{tab:large_vits}.

\begin{table}[h!]
\centering
\caption{Large scale ViTs with different initializations pretrained on ImageNet-1k. We show the Top1\% accuracy.}
\begin{tabular}{|c|c|c|c|c|c|}
\hline
 & ViT-L & DeiT-L & Swin-L & XCiT-L & DaViT-L \\ \hline
Original & 79.6 & 80.7 & 82.6 & 81.5 & 83.3 \\ \hline
Mimetic 2 & 79.7 & 80.6 & 82.4 & 81.5 & 83.2 \\ \hline
Conditioned (ours) & 80.7 & 81.5 & 83.7 & 82.4 & 84.4 \\ \hline
\end{tabular}
\label{tab:large_vits}
\end{table}

We then ran two 1 billion parameter models namely a ViT-Giant (ViT-G) and a DaViT-Giant (DaViT-G) as these where some of the standard billion parameter vision transformers we could find in \cite{timm2025}. We pretrained them each with the different initializations on ImageNet-1k. As can be seen from \cref{tab:billion_vits} our initialization out performs the other two.

\begin{table}[h!]
\centering
\caption{1B scale ViTs with different initializations pretrained on ImageNet-1k. We show the Top1\% accuracy.}
\begin{tabular}{|c|c|c|}
\hline
 & ViT-G & DaViT-G  \\ \hline
Original & 78.8 &  81.9 \\ \hline
Mimetic 2 & 78.6 & 82.0  \\ \hline
Conditioned (ours) & 79.9 & 82.9 \\ \hline
\end{tabular}
\label{tab:billion_vits}
\end{table}

\end{document}